\documentclass{article}

\usepackage{amsmath} 
\usepackage{amssymb,amsthm} 

\usepackage{multirow}

\usepackage{bbm}

\usepackage{url}

\newtheorem{lemma}{Lemma}
\newtheorem{theorem}{Theorem}

\newtheorem{remark}{Remark}
\newtheorem{definition}{Definition}

\newtheorem{assumption}{Assumption}

\usepackage{array}
\newcolumntype{P}[1]{>{\centering\arraybackslash}p{#1}}

\usepackage[title]{appendix}

\usepackage{microtype}
\usepackage{graphicx}
\usepackage{subfigure}
\usepackage{booktabs} 

\usepackage{hyperref}

\usepackage[accepted]{icml2020}

\icmltitlerunning{Hybrid-Order Distributed SGD to Balance Communication Overhead, Computational Complexity, and Convergence Rate}

\begin{document}

\twocolumn[
\icmltitle{A Hybrid-Order Distributed SGD Method for Non-Convex Optimization to Balance Communication Overhead, Computational Complexity, and Convergence Rate}



\icmlsetsymbol{equal}{*}

\begin{icmlauthorlist}
\icmlauthor{Naeimeh Omidvar}{ipm}
\icmlauthor{Mohammad~Ali Maddah-Ali}{nokia}
\icmlauthor{Hamed Mahdavi}{ce_sharif}
\end{icmlauthorlist}

\icmlaffiliation{ipm}{School of Computer Science, Institute for Research in Fundamental Sciences (IPM), Tehran, Iran (email: nomidvar@connect.ust.hk).}
\icmlaffiliation{nokia}{Nokia Bell Labs (email: mohammad.maddah-ali@nokia-bell-labs.com).}
\icmlaffiliation{ce_sharif}{Department of Computer Engineering, Sharif University of Technology, Tehran, Iran (email: hmdmahdavi@ce.sharif.edu)}


\icmlkeywords{Machine Learning, ICML}

\vskip 0.3in
]



\printAffiliationsAndNotice{}  

\begin{abstract}

In this paper, we propose a method of distributed stochastic gradient descent (SGD), with low communication load and computational complexity, and still fast convergence. To reduce the communication load, at each iteration of the algorithm, the worker nodes calculate and communicate some scalers, that are the directional derivatives of the sample functions in some \emph{pre-shared directions}. However, to maintain accuracy, after every specific number of iterations, they communicate the vectors of stochastic gradients. To reduce the computational complexity in each iteration, the worker nodes approximate the directional derivatives with zeroth-order stochastic gradient estimation, by performing just two function evaluations rather than computing a first-order gradient vector. The proposed method highly improves the convergence rate of the zeroth-order methods, guaranteeing order-wise faster convergence. Moreover, compared to the famous communication-efficient methods of model averaging (that perform local model updates and periodic communication of the gradients to synchronize the local models), we prove that for the general class of non-convex stochastic problems and with reasonable choice of parameters, the proposed method guarantees the same orders of communication load and convergence rate, while having order-wise less computational complexity. Experimental results on various learning problems in neural networks applications demonstrate the effectiveness of the proposed approach compared to various state-of-the-art distributed SGD methods.

\end{abstract}

\section{Introduction}

Stochastic gradient descent (SGD) is an optimization tool 
 which is widely used  for solving many machine learning problems, 
due to its simplicity 
 and acceptable empirical performance 
  \cite{rakhlin2011making, bottou2010large}. 
 It iteratively updates the model parameters at the opposite direction of the stochastic gradient of the cost function. 
With the emergence of large-scale machine learning models such as deep neural networks (DNNs), however, centralized deployment of SGD has become intractable in its memory and time requirements. 
As such, and accelerated with the recent advances in multi-core parallel processing technology, distributed deployment of SGD 
 has become a trend, which trains machine learning models 
 on multiple computation nodes (a.k.a., worker nodes) in parallel, to enhance the scalability of the training procedure   \cite{meng2019convergence, zinkevich2010parallelized}. 
Under a distributed SGD method, at each iteration, each 
worker node evaluates a stochastic estimation of the gradient of the objective 
function over a randomly chosen/arrived data sample. The workers  
then broadcast their gradient updates 
to their peers, and aggregate 
 the gradients to update the model in parallel. 

Some of the major factors in designing distributed SGD algorithms, which are in conflict with each other and need to be addressed carefully, can be identified as follows:

\textbf{1- Communication overhead:} 
Distributed implementation of SGD algorithms mainly suffers from high 
level of communication overhead due to exchanging  
the stochastic gradient vectors among the worker nodes at each iteration, especially when the dimension of the model is large.  
In fact, the communication overhead has been observed to be 
the major bottleneck in scaling distributed SGD, and the time required to exchange this information will increase the overall time required for the algorithm to converge \cite{zhou2019distributed, alistarh2017qsgd, strom2015scalable, chilimbi2014project}.  
As such, alleviating the communication cost has recently gained lots of attention from different research communities. 

To address the issue of communication overhead, two main approaches can be identified 
in the literature: 1- 
Reducing the number of rounds of communication among the worker nodes by 
allowing them to perform local model updates at each iteration, and limiting their synchronization 
to a periodic exchange of local models after some 
iterations. This approach is often referred to as model averaging \cite{mcdonald2010distributed, zinkevich2010parallelized, zhang2016parallel, su2015experiments, zhang2015deep, povey2014parallel, kamp2018efficient, wang2018cooperative, zhou2017convergence}, 2- Reducing the number of bits communicated among the worker nodes 
by compression or quantization of the gradient vectors \cite{alistarh2017qsgd, wen2017terngrad, dean2012large, abadi2016tensorflow, yu2014introduction, de2015taming}. This approach  propagates the quantization error through the iterations, reduces the convergence rate, and would add to the computational complexity of each iteration.

\textbf{2- Computation load:} In many practical applications, the worker nodes are generally commodity devices with limited computational resources. 
Therefore, in order to guarantee a scalable solution, 
the distributed algorithm should essentially 
impose  as low computational cost 
per worker node as possible. 

Note that in general, one of the major derives for the computational load in SGD algorithms is the calculation of the first-order gradients at each iteration, which is very expensive, if not impossible, to obtain in many real-world problems (especially for learning large-scale models) \cite{hajinezhad2017zeroth}. In many applications, calculating  
the stochastic gradient requires extensive simulation for each worker node, where the complexity of each simulation may require significant computational time \cite{fu2015stochastic}. 
Moreover, in many scenarios of training 
DNNs,  
there is a highly complicated relationship between the 
parameters of the model and the objective function, so that deriving explicit form of the gradient is impossible and computing the gradient at each iteration costs significantly high computational complexity for the worker node \cite{lian2016comprehensive}.{\footnote{It is noted that although 
under fast differentiation techniques, 
the gradient of the sample function in neural network can be derived with less complexity, there exist some restrictions for applying such technique to a general neural network. In particular, such techniques require to store the results of all the intermediate computations, which is 
impractical 
for many scenarios due to memory limitations \cite{nesterov2017random}.}} 

We note that a promising approach to reduce the computational load of SGD methods is utilizing zeroth-order (ZO) gradient estimations instead of deriving the first-order gradients. 
In fact, since calculating a ZO gradient requires just two function evaluations \cite{sahu2019towards, ji2019improved}, it highly reduces the computational cost compared to  
the calculation of a first-order gradient (which imposes $ \mathcal{O} \left(d\right) $ times higher  computational complexity in general \cite{nesterov2017random}).  
Furthermore, by using pre-shared seeds for the generation of the random directions involved in zeroth-order gradients calculation, the worker nodes will not need to communicate the whole ZO gradient vectors, and can send just the scalar values of the  computed 
derivative approximations, as will be shown in Section \ref{sec: proposed method}. As such, the communication overhead is significantly reduced as compared to the communication overhead imposed by sending the whole gradient vectors, 
especially when the dimension of the problem is large. 
However, loosing accuracy in calculating the gradients directly reflects in the convergence rates of zeroth-order SGD methods, especially for non-convex problems (e.g., see the convergence rates comparison in Table \ref{tab: comparison}). This is one of the main 
challenges that we overcome in this work. 

\textbf{3- Convergence rate:} 
Finally achieving fast convergence to the solution of a problem is essentially desired, especially for large-scale problems dealing with huge datasets. 

\subsection{Main Contributions}

In this work, we propose a distributed 
optimization method that 
strikes a better balance between communication efficiency, computation efficiency and accuracy, compared to various distributed methods. The proposed methods enjoys the low computational complexity and communication overhead of the zeroth-order gradient calculation, and at the same time, benefits from periodic first-order gradient calculation and model updating.  
We theoretically prove that by proper combining of these two, we can achieve 
very good 
convergence rate as well. 
The main contributions of this work can be summarized as follows:

\begin{center}

\begin{table*}[t]\caption{Comparison of the proposed method to 
various state-of-the-art methods in the literature.}
\label{tab: comparison}
\centering
{\small{
\begin{tabular}
{|p{4.3 cm}||p{2.5 cm}|P{2.4 cm}|P{2.2 cm}|p{3.5
 cm}|} 
    \hline
 {Method}& {Convergence Order}   & {Communication Load per Iteration}& {Normalized Computational Load} & {Comments} \\  \hline\hline
  Proposed & $ \mathcal{O} (\frac{d}{\sqrt{mN}} ), \text{If}~\tau > 1   $ {\small{$ \mathcal{O} ( \frac{1}{\sqrt{mN}} ), \text{If}~\tau = 1  $}}
     & $ \frac{\tau -1+d}{\tau} $ & $  \simeq \frac{1}{\tau}+\frac{1}{d}  $  &  \\ \hline
  RI-SGD 
   \cite{haddadpour2019trading} 
   & $ \mathcal{O} (\frac{\tau}{\sqrt{mN}}  ) $ & $ \frac{d}{\tau} $  & $ \mu m + 1 $ & requires high storage, $ \mu $: redundancy factor  \\ \hline
  syncSGD 
   \cite{wang2018cooperative} & $  \mathcal{O}  ( \frac{1}{\sqrt{mN}}  )  $  & $ d  $  & $ 1 $   &   \\ \hline   
  ZO-SGD 
   \cite{sahu2019towards} & $ \mathcal{O}  (\frac{ (d/m )^{1/3}}{ ( N  )^{1/4}}  ) $  & $ 1 $ & $  \simeq \frac{1}{d} $    &   \\ \hline
 ZO-SVRG-Ave 
 \cite{liu2018zeroth} & $ \mathcal{O}  (\frac{d}{N}+\frac{1}{ \min{\{d,m\}}}  ) $  & $ 1 $ & $ \mathcal{O} (\frac{K}{d}) $   &  requires dataset storage, $ K $: size of dataset   \\ \hline  
QSGD 
 \cite{alistarh2017qsgd}  &  $ \mathcal{O}  (\frac{1}{N}+\sqrt{d}  ) $ 
 &  $ \mathcal{O}  (s^2 + s \sqrt{d}  ) $  &  $ 
 > 1 $   & $ s $: num. of quantization levels  \\ \hline  
\end{tabular}
}}
\end{table*} 

\end{center}

\begin{itemize}

\item We develop a new distributed SGD method, with low communication overhead and computational complexity, and yet fast convergence. To reduce the communication load, at each iteration of the algorithm, the worker nodes calculate and communicate some scalers, that are the directional derivatives of the sample functions in some \emph{pre-shared directions}. To reduce the computational complexity, the worker nodes approximate the directional derivatives with zeroth-order stochastic gradient estimation, by performing just two function evaluations.  
Finally, to alleviate the  approximation error of the zeroh-order stochastic gradient estimations, 
after every $ \tau  \in \mathbb{N}$ iterations, the worker nodes compute and communicate the first-order stochastic gradient vectors.

\item 
We provide theoretical analyses for the convergence rate guarantee of the proposed method. 
In particular, we prove that for the general class of non-convex stochastic problems, the proposed scheme converges to a stationary point with the rate of $ \mathcal{O} ( d/\sqrt{mN} ) $, which highly outperforms the convergence rates  of zeroth-order methods (e.g., see \cite{sahu2019towards, liu2018zeroth}). 
Moreover, compared to the fastest first-order communication-efficient algorithms (e.g., 
the model averaging scheme in \cite{haddadpour2019trading} with the convergence rate of $ \mathcal{O} ( \frac{\tau}{\sqrt{m N}}  ) $), the proposed algorithm exhibits the same convergence rate in terms of both the number of iterations and the number of worker nodes.  Finally, by a reasonable choice of  $ \tau = \mathcal{O} \left( d \right) $, we can guarantee 
the same-order communication load and convergence rate as in the fastest converging 
model averaging methods, with order-wise less computational complexity (see Table \ref{tab: comparison}).

\item Due to utilizing zeroth-order stochastic gradient updates, the proposed algorithm exhibits a sufficiently low level of computational complexity, which is 
 much lower than the one in the communication-efficient methods including the model averaging schemes (e.g., \cite{haddadpour2019trading, wang2018cooperative}) or the schemes with gradients compression (e.g., \cite{alistarh2017qsgd}), and is comparable to the one in the existing zeroth-order stochastic optimization algorithms (e.g., \cite{sahu2019towards, liu2018zeroth}). 
As a baseline, computational complexity of the proposed method is $ \mathcal{O} \left( 1/\tau+1/d \right) $ times the computational complexity of the model averaging schemes.

\item 
Due to utilizing pre-shared seeds for the generation of random directions, at the iterations with zeroth-order gradient updates, the worker nodes just need to communicate a scalar rather than a $ d $-dimensional vector, 
which drastically reduces the number of bits required for communication among the worker nodes.  
In particular,  
the proposed algorithm exhibits a communication load equal to sending  
$ d + \tau -1 $ scalar values by each worker node per  $ \tau $ iterations. This communication load is much lower than that of 
fully synchronous SGD (syncSGD)  \cite{wang2018cooperative}, and is comparable to that of the communication-efficient schemes with model averaging, with a ratio of $ \left( 1+\tau/d \right) $.

\item Using numerical experiments, we empirically demonstrate the accuracy and convergence properties of the proposed method compared to various  state-of-the-art baselines. 

\end{itemize}
 
Table \ref{tab: comparison} summarizes  
the comparison of the proposed method to the most related works in the literature,  
in terms of convergence rate, communication overhead per iteration, and computational load per iteration normalized to the computational complexity of  
computing a first-order stochastic gradient.

\section{Related Work}

As mentioned before, communication bottlenecks in distributed SGD originate from two sources that can be significant; first, the number of bits communicated at each communication round, and second, the number of rounds of communication. 
To tackle the first barrier, the existing works mainly try to quantize and compress the gradient vectors before communicating them \cite{alistarh2017qsgd, wen2017terngrad, dean2012large, abadi2016tensorflow, yu2014introduction, de2015taming}. 
Although those methods are generally effective in reducing the number of bits communicated at each iteration, 
their quantization error increases the error variance of the communicated gradient vectors, leading to a slow convergence. 
To tackle the second  
source of communication overhead, 
the idea of model averaging has been proposed, where  the worker nodes perform local updates at each iteration, and communicate their updated 
models after every $ \tau $ iterations 
 to periodically synchronize them
 \cite{mcdonald2010distributed, zinkevich2010parallelized, zhang2016parallel, su2015experiments, zhang2015deep, povey2014parallel, kamp2018efficient, wang2018cooperative, zhou2017convergence}. In consequence, the number of rounds of communication is significantly reduced.

There exist various model averaging schemes in the literature. 
\cite{mcmahan2016communication}  
investigated periodic averaging SGD (PA-SGD), where the models are averaged across the nodes afters every $ \tau $ local updates. 
\cite{wang2018cooperative} 
 demonstrated that the convergence error of PA-SGD grows linearly in the number of local updates $ \tau $. \cite{yu2019parallel} provided some theoretical studies on 
why SGD with model averaging works well.  
Moreover, \cite{zhang2015deep} 
proposed elastic averaging SGD (EASGD), where 
a more complicated averaging is done by encompassing 
 an auxiliary variable 
to allow some slackness between the models. 
They empirically validated the effectiveness of EASGD, without providing any convergence analysis. 
Furthermore,  \cite{jiang2017collaborative, lian2017can} 
 proposed consensus-based distributed SGD (D-PSGD), in which worker nodes 
 synchronize their local models only with their neighboring nodes. 
By incorporating extra memory as well as high storage, \cite{haddadpour2019trading} infuse redundancy to the training data to further reduce the residual error in local averaging, and improve the convergence rate. 
However, this method needs more storage per worker node, and the data at different worker nodes are overlapping. 

Finally, note that a key feature for both the aforementioned 
communication-efficient approaches is that they require  
stochastic first-order gradient information at each iteration (via subsequent calls to a stochastic first-order oracle (SFO)), in order to guarantee the convergence. However, as illustrated before, 
 in many real-world applications and scenarios, 
it may be computationally costly 
 for the worker nodes to obtain such information at every iteration \cite{hajinezhad2017zeroth}. 
Therefore, although the approaches mentioned above 
provide communication-efficient solutions to distributed learning, 
their computational complexity may not be 
tolerable for general commodity worker nodes, 
and restricts the applicability of such methods, in practice. 

To reduce the computational load of SGD methods, the idea of zeroth-order gradient estimations can be utilized.  
This idea has been widely used for gradient-free optimization where the explicit expressions of gradients of the objective function are expensive or infeasible to obtain, and only function evaluations are accessible \cite{sahu2019towards, ji2019improved}. As earlier discussed in the Introduction section, ZO gradient estimation can highly reduce the communication overhead as well as the computational complexity. However, the aforementioned benefits of zeroth-order SGD methods   
comes at the cost of significantly deteriorated convergence rates. 
This is because a zeroth-order gradient is in fact a biased estimation of the true gradient, and the involved approximation error leads to a high residual error and consequently,  
inferior convergence rates in zeroth-order SGD algorithms 
\cite{nesterov2017random}. 
For example, for general non-convex problems, the (centralized) zeroth-order SGD algorithm proposed in  \cite{sahu2019towards} has a 
convergence rate of $ \mathcal{O} (  \frac{ d^{1/3} }{ m^{1/3} N^{1/4} } ) $ (where  $ m $  denotes the number of sampled directions at each iteration of the algorithm, and hence, can be considered as equivalent to the number of worker nodes for a distributed setting). 

By incorporating multiple restarts and extra memory, \cite{liu2018zeroth} proposed a zeroth-order extension of stochastic variance reduced gradient method (SVRG) and proved that their proposed zeroth-order SVRG (ZO-SVRG-Ave) achieves a convergence rate of 
$ \mathcal{O}  (\frac{d}{N}+\frac{1}{ \min{\{d,m\}}}  ) $ 
for non-convex problems, in which the second term 
highly deteriorates the convergence performance of the algorithm. 
To eliminate this error term, they proposed a coordinate-wise version of ZO-SVRG, but it costs $ \mathcal{O} (d) $ times more function evaluations 
per iteration, which causes high computational complexity, especially for high-dimensional problems. 
Moreover, both of their proposed ZO-SVRG methods required full dataset storage, which is not affordable for distributed deployment on commodity devices. 
Furthermore, \cite{gao2018information} proposed a zeroth-order version of the stochastic alternating direction method of multipliers with a guaranteed  convergence rate of $ \mathcal{O} (1/\sqrt{N}) $ for convex problems but cannot be applied to general non-convex problems. 
We note that the slow convergence of zeroth-order stochastic gradient methods, especially for non-convex problems, is one of the main challenges that we address in this paper.

\section{Communication-Efficient Distributed Stochastic Optimization Algorithm}\label{sec: proposed method}

\subsection{Problem Definition}

Consider a setting with $ m $ distributed worker nodes that are interested in solving the following non-convex stochastic optimization problem in a distributed manner:
\begin{equation}\label{eq: prob formulation}
\min_{\boldsymbol{x} \in \mathbb{R}^d } f \left( \boldsymbol{x} \right) = \min_{\boldsymbol{x} \in \mathbb{R}^d} \mathbb{E}_{\boldsymbol{\zeta}} \left[ F \left( \boldsymbol{x} , \boldsymbol{\zeta} \right) \right], 
\end{equation}
where $ f  $ is a generic non-convex loss function, and $ \boldsymbol{\zeta} $ is a 
random variable with unknown distribution. 
As a special case of the above 
formulation,  
consider the following non-convex  
sum 
problem  
which appears in a variety of machine learning applications, ranging from generalized linear models to deep neural networks: 
\begin{equation}\label{eq: prob form 2}
\min_{\boldsymbol{x} \in \mathbb{R}^d } f \left( \boldsymbol{x} \right) = \min_{\boldsymbol{x} \in \mathbb{R}^d} \dfrac{1}{K}\sum_{k=1}^K F \left( \boldsymbol{x} , \boldsymbol{\zeta}_k \right) , 
\end{equation}
in which $ F \left( \boldsymbol{x} , \boldsymbol{\zeta}_k \right) $ is the  
training loss over the
sample $ \boldsymbol{\zeta}_k $, 
 and $ K $ is the total number of samples. 
Note that for generality, in the rest of this paper, we focus on the problem formulation \eqref{eq: prob formulation}. 
However, our analyses apply equally to both objectives.

\subsection{The Proposed Algorithm}

In the proposed algorithm, named as Hybrid-Order Distributed SGD (HO-SGD), at each iteration $ t $, each worker node $ i $ 
samples or receives data sample $ \boldsymbol{\zeta}_{t+1,i} $ identically and independently from the dataset. 
Then, the data sample is used to evaluate 
a stochastic zeroth-order gradient 
approximation 
via computing a finite difference using two 
function queries 
as follows \cite{liu2018zeroth, gao2018information}. 
\begin{align}
\tilde{\boldsymbol{G}}_{t,i} \hspace{-1 pt} = \hspace{-1 pt}  
\dfrac{d}{\mu} \hspace{-2 pt} \Bigg[ \hspace{-1 pt} F \hspace{-1 pt} \left( \boldsymbol{x}^t \hspace{-1 pt} + \hspace{-1 pt} \mu \boldsymbol{v}_{t+1,i} , \boldsymbol{\zeta}_{t+1,i} \right) \hspace{-3 pt} - \hspace{-3 pt} F \hspace{-1 pt} \left( \boldsymbol{x}^t , \boldsymbol{\zeta}_{t+1,i} \right) \hspace{-3 pt} \Bigg] \boldsymbol{v}_{t+1,i}, \notag 
\end{align}
where $ d $ is 
the dimension of the model to be learned, $ \mu > 0 $ is a smoothing parameter, and $ \boldsymbol{v}_{t+1,i} $ is a  
direction randomly drawn from a uniform distribution over a unit sphere. 

Next, each worker node $ i $ 
communicates its computed zeroth-order gradient to the other nodes. For this purpose, note that  
the 
direction of the derivatives, i.e., the vectors $ \boldsymbol{v}_{t+1,i} , ~ \forall i $, are some randomly generated directions, where the seeds are pre-shared among the nodes before optimization. As such, each worker node $ i $ does not need to send the vector and just needs to send  the value of the finite-difference approximated directional derivative, i.e., the scalar $ \dfrac{d}{\mu} \left[ F\left( \boldsymbol{x}^t + \mu \boldsymbol{v}_{t+1,i} , \boldsymbol{\zeta}_{t+1,i} \right) - F\left( \boldsymbol{x}^t , \boldsymbol{\zeta}_{t+1,i} \right) \right] $.  
Therefore, instead of communicating a $ d $-dimensional gradient vector, each node  communicates just a single scalar. As a result, the communication load reduces to $ d $ times less than the communication load of exchanging the 
stochastic gradient vector. 
The model is then updated by each worker node in parallel, using the average of the local zeroth-order gradients of all the nodes, as shown by \eqref{eq: update G tile t}-\eqref{eq: update x}.

Finally, after every $ \tau -1 $ iterations, the worker nodes perform one iteration of first-order stochastic gradient computation and communication, and update the model accordingly. 
The pseudo-code of the proposed algorithm is shown in Algorithm \ref{alg: proposed alg}. 
We note that the introduced algorithm does not necessarily assume 
that each worker node has access to 
the entire data. 
Rather, as long as each data sample is assigned to each worker node uniformly at random, Algorithm \ref{alg: proposed alg} works, and all the  results apply.

\subsection{Discussion on the Proposed Algorithm and Comparison to the Related Works}  

In the following, we review some remarks regarding  Algorithm \ref{alg: proposed alg}. 
First, note that if we choose   $ \tau=1 $, Algorithm \ref{alg: proposed alg} reduces to fully synchronous 
distributed SGD method  \cite{wang2018cooperative, dekel2012optimal}, where the workers perform first-order gradient computation and communication at all iterations. Moreover, if we consider $ \tau \geq N $, the workers always perform zeroth-order gradient updates, thereby  
the algorithm reduces to distributed zeroth-order stochastic gradient method. 
Therefore, as the two ends of its spectrum, the proposed algorithm encompasses both zeroth-order and first-order distributed SGD algorithms as special cases.

\begin{algorithm}[H] 
   \caption{Hybrid-Order Distributed SGD Algorithm}
   \label{alg: proposed alg}
\begin{algorithmic}
   \STATE {\bfseries Input:} Dimension $ d $, the total number of iterations $ N $, the number of workers $ m $, period $ \tau \in \mathbb{N} $, smoothing parameter $ \mu $, batch size $ B $, initial point $ \boldsymbol{x}^0 $, step-size rule $ \left\{ \alpha_t \right\}_{t=0}^{N-1} $. 
   
   \FOR{$ t=0,\ldots, N-1 $ }
   \STATE \textbf{parallel for} $ i=1,\ldots, m $ \textbf{do}
   \STATE \hspace{10 pt} $i$-th worker receives a batch of i.i.d. 
   samples $ \left \lbrace \boldsymbol{\zeta}_{t+1,i,b } \right \rbrace_{b=1,\ldots,B} $. 
   
   \hspace{10 pt}\IF{$ \mod(t ,\tau)=0 $} 
   
	\STATE \hspace{10 pt} $i$-th worker computes 
	the first-order 
	stochastic gradient vector, as follows: 
	\begin{align}\label{eq: update G_1 tile t i}
	\tilde{\boldsymbol{G}}_{t,i} \gets &  
	\dfrac{1}{B} \sum_{b=1}^B \nabla F \left( \boldsymbol{x}^t, \boldsymbol{\zeta}_{t+1,i ,b } \right)   
	\end{align}
   \vspace{-15 pt}
   \hspace{10 pt}\ELSE
   
   \STATE \hspace{10 pt} $i$-th worker picks a direction $ \boldsymbol{v}_{t+1,i} $ uniformly at random from the unit sphere, 
   and computes a zeroth-order stochastic gradient as follow:
   %
	\begin{align}\label{eq: update G_0 tile t i}
	\tilde{\boldsymbol{G}}_{t,i} \gets  
	\dfrac{1}{B} \sum_{b=1}^B \dfrac{d}{\mu}   \Big[  & F\left( \boldsymbol{x}^t + \mu \boldsymbol{v}_{t+1,i} , \boldsymbol{\zeta}_{t+1,i ,b } \right)  \notag \\
	& - F\left( \boldsymbol{x}^t , \boldsymbol{\zeta}_{t+1,i ,b } \right) \Big] \boldsymbol{v}_{t+1,i}   	 
	\end{align}
   
   \vspace{-10 pt}
   \hspace{10 pt}\ENDIF
   \vspace{-15 pt}
   
   \STATE 
   \begin{equation}\label{eq: update G tile t}
   \tilde{\boldsymbol{G}}_t = \dfrac{1}{m} \sum_{i=1}^m \tilde{\boldsymbol{G}}_{t,i}
   \end{equation}
   \begin{equation}\label{eq: update x}
   \boldsymbol{x}^{t+1} = \boldsymbol{x}^t - \alpha_t \tilde{\boldsymbol{G}_t}
   \end{equation}

   \STATE \textbf{end parallel for}
   \ENDFOR
   
	\STATE {\bfseries Output:} $ \boldsymbol{x}^N $.
\end{algorithmic}
\end{algorithm}

To address the challenge of slow convergence of zeroth-order SGD iterations, the proposed method employs periodic rounds of first-order stochastic gradient updates. This  significantly reduces the residual error of the zeroth-order stochastic gradient updates, as compared to the existing zeroth-order methods. 
Therefore, our algorithm can be viewed as a zero-order stochastic optimization with periodic rounds of first-order updates, 
which reduces the communication overhead and  computational complexity, and at the same time guarantees  improved convergence rate (as will be seen in the next section). 
In particular, the main advantages of the proposed algorithm can be identified as follows:

\textbf{Low computational complexity:} By employing 
zeroth-order stochastic gradient estimations, each worker node 
performs just two function evaluations,  rather than a  
gradient computation. 
 This contributes to a significant reduction in the computational complexity of the proposed algorithm, as compared to the previous first-order distributed methods. 
In particular, it is estimated that in general, computing a zeroth-order gradient estimation costs $ \mathcal{O} \left( d \right) $ times less computational load than computing a first-order gradient estimation \cite{nesterov2017random}. 
Therefore, considering $ \tau -1 $ iterations of zeroth-order update and one iteration of first-order update at each period of $ \tau $ iterations of  Algorithm \ref{alg: proposed alg}, the computational complexity of the proposed algorithm is extremely lower than that of the first-order 
 communication-efficient methods, 
with a ratio of $ \mathcal{O} \left( \frac{1}{d}+\frac{1}{\tau} \right) $, and is comparable to the one in the distributed zeroth-order methods.

\textbf{Low communication overhead:} As aforementioned, due to utilizing pre-shared seeds for the generation of random directions, at the iterations with zeroth-order gradient updates, the worker nodes just need to communicate a scalar instead of a $ d $-dimensional vector, 
which drastically reduces the number of bits required for communication. 
In particular, 
the communication overhead of the proposed scheme is $ \left( 1 + \frac{\tau-1}{d} \right) $ times the communication overhead of the fastest first-order communication-efficient methods (i.e., model-averaging schemes such as \cite{haddadpour2019trading}), and is comparable to the communication overhead in the zeroth-order methods. 
 
\textbf{Fast convergence rate:} Finally, 
the proposed algorithm guarantees a fast convergence rate for a general class of non-convex problems. As will be shown in the next section, in terms of the number of iterations and the number of workers, the proposed method guarantees the same convergence rate as those of the fastest first-order communication-efficient methods, and order-wisely better convergence rate than those of the zeroth-order methods.

\section{Convergence Analysis} \label{sec: convergence analysis}

In this section, we present the convergence analysis of the proposed algorithm for a general class of non-convex problems.  
Prior to that, we first state the  
main assumptions and definitions used for the convergence analysis.

\subsection{Assumptions and Definitions}

Our convergence analysis is based on the following assumptions, which are all standard and 
widely used in the context of non-convex optimization \cite{meng2019convergence}. 

\begin{assumption}[Unbiased and finite variance first-order stochastic gradient estimation]\label{assump: unbiased 1st grad estimation + bounded var}
The stochastic gradient evaluated on each data sample $ \boldsymbol{\zeta} $ by the first-order oracle is an unbiased estimator of the full (exact) gradient, i.e., 
\begin{equation}\label{eq: 1st grad unbiased}
\mathbb{E} \left[ \nabla F \left( \boldsymbol{x},\boldsymbol{\zeta} \right) \right] = \nabla f \left( \boldsymbol{x} \right), ~~\forall \boldsymbol{x} \in \mathbb{R}^d, 
\end{equation} 
with a finite variance, i.e., there exists a constant $ \sigma \leq 0 $ such that 
\begin{equation}\label{eq: 1st grad bounded var}
\mathbb{E} \left[ \parallel \nabla F \left( \boldsymbol{x},\boldsymbol{\zeta} \right) - \nabla f \left( \boldsymbol{x} \right) \parallel^2  \right] \leq \sigma^2 , ~~\forall \boldsymbol{x} \in \mathbb{R}^d.  
\end{equation} 
\end{assumption}

\begin{assumption}[Lipschitz continuous and bounded gradient]\label{assump: Lipschitz grad + bounded grad}
The objective function $ isf \left( \boldsymbol{x} \right) $ is differentiable and $ L $-smooth, i.e., its gradient $ \nabla f $ is $ L $-Lipschitz continuous: 
\begin{equation}\label{eq: Lips grad}
\parallel \nabla f \left( \boldsymbol{x} \right) - \nabla f \left( \boldsymbol{y} \right) \parallel \leq L \parallel  \boldsymbol{x} - \boldsymbol{y} \parallel, ~~~\forall \boldsymbol{x},\boldsymbol{y} \in \mathbb{R}^d .  
\end{equation} 
Moreover, the norm of the gradient of the objective function is bounded, i.e., there exists a constant $ M \leq 0 $ such that 
\begin{equation}\label{eq: bounded grad}
\parallel \nabla f \left( \boldsymbol{x} \right) \parallel \leq M , \quad \forall \boldsymbol{x} \in \mathbb{R}^d.  
\end{equation} 
\end{assumption}

\begin{assumption}[Bounded bellow objective function value]\label{assump: lower bound f*}
The objective function value is bounded below by a scalar $ f^\ast $.
\end{assumption}

\subsection{Main Results}

First, note that since  $ f \left( \boldsymbol{x} \right) $ is non-convex, we need a proper measure to show  
 the gap between the the output of the algorithm 
and the set of stationary solutions. 
As such and similar to the previous works, we consider the expected gradient norm as an indicator of convergence, 
and state the algorithm achieves an $ \epsilon $-suboptimal solution if 
\cite{bottou2018optimization}:  
\begin{equation}
\mathbb{E} \left[ \dfrac{1}{N} \sum_{t=1}^N \parallel \nabla f \left( \boldsymbol{x}^t \right) \parallel^2 \right] \leq \epsilon.
\end{equation} 
Noted that this condition guarantees convergence of the algorithm to a stationary point \cite{wang2018cooperative}. 

The following theorem presents the main convergence result of the proposed algorithm. 
Note that all the intermediate results and analytical proofs have been deferred to Appendix \ref{sec: app proofs}. 

\begin{theorem}[Convergence of the proposed method]\label{th: main}
In the proposed Algorithm 1, under Assumptions \ref{assump: unbiased 1st grad estimation + bounded var}-\ref{assump: lower bound f*}, if the step-size (a.k.a., the learning rate) and the smoothing parameter are chosen 
such that $ \alpha_t = \frac{\sqrt{ B m}}{L \sqrt{N}}, ~ \forall t $ and $ \mu \leq \frac{1}{\sqrt{dN}} $, respectively, 
and the total iterations $ N $ is sufficiently large, i.e., 
$ N > \frac{16 \left( d+ B m-1 \right)^2}{B m} $, 
then the average-squared gradient norm after $ N $ iterations is bounded by: 
\begin{align}\label{eq: error bound in the main theorem}
\dfrac{1}{N} & \sum_{t=0}^{N-1}  \mathbb{E} \left[ \parallel \nabla f \left( \boldsymbol{x}^{t} \right) \parallel^2 \right]   \\
& \leq  \frac{4 L \left( f \left( \boldsymbol{x}^0 \right) - f^\ast \right) }{\sqrt{ B m N}}  
+  \frac{ 2 \sigma^2 }{ \sqrt{ B mN} \tau} \notag \\
& \quad + 1 \left(\tau > 1 \right)  \Bigg( \frac{4 L^2}{ d^2 \sqrt{ B m N} \tau} 
+ \frac{4 L^2}{ d^2 N \sqrt{ B m N} } \notag \\
& \quad + \frac{L^2}{\sqrt{ B mN}} \frac{\tau-1}{\tau} 
+ \frac{ L^2 }{N \sqrt{ B mN} \tau}
+ \frac{4 d \sigma^2}{ \sqrt{ B m N} } \frac{\tau-1}{\tau}  \notag \\
& \quad   + \frac{4 d \sigma^2}{ N \sqrt{ B m N} \tau}  
+ \frac{L^2 }{ \sqrt{ B mN} } \dfrac{\tau-1}{\tau} 
+ \frac{ L^2 }{ N \sqrt{ B mN} \tau} \Bigg),
 \notag
\end{align} 
where $ \mathbf{1} (\cdot) $ is the indicator function. 

\end{theorem}

\begin{remark}[\textbf{Order of Convergence}]\label{rem: convergence order}
The above theorem indicates that 
for any $ \tau > 1 $, the convergence rate is of the order of 
$ \left(\frac{d}{\sqrt{mN}} \frac{\tau - 1}{\tau} \right) = \mathcal{O} \left(\frac{d}{\sqrt{mN}} \right) $
with respect to the parameters of the problem, which significantly outperforms the zeroth-order stochastic gradient methods (see Table \ref{tab: comparison} for more details). 
Moreover, by setting the period of the first-order gradient updates such that $ \tau = \mathcal{O} \left( d \right) $, we can guarantee the same convergence rate as in the fastest converging communication-efficient methods such as \cite{haddadpour2019trading}, which has the convergence rate of $ \mathcal{O} \left(\frac{\tau}{\sqrt{mN}} \right) $.  
Also note that in the special case of  $ \tau \geq N $, i.e., when the workers only use stochastic zeroth-order oracle for gradient estimation, the proposed algorithm recovers zeroth-order SGD algorithm. 
In this case also, our obtained convergence rate significantly improves the previous results on the convergence rate of the zeroth-order SGD algorithms  \cite{sahu2019towards, liu2018zeroth}. 
Finally, when 
$ \tau = 1 $, i.e., when the workers only use stochastic first-order oracle for gradient estimation, the above theorem guarantees a convergence rate of $ \mathcal{O} (\frac{1}{\sqrt{mN}} ) $, 
which is consistent with 
the convergence rate of fully synchronous SGD \cite{wang2018cooperative}, \cite{dekel2012optimal}.  

\end{remark}

\begin{remark}[\textbf{Error decomposition}]\label{rem: error decomp}
It is noted that the error upper bound \eqref{eq: error bound in the main theorem} is decomposed into two parts. The first part, which contains the first two terms,  
is similar to  
 the error bound in fully synchronous SGD \cite{bottou2018optimization}, and is resulted from the periodic steps of first-order stochastic gradient updates. The second part, containing the remaining terms,  
is a measure of the approximation error in zeroth-order stochastic gradient steps, and vanishes when $ \tau =1 $, i.e., when there is no zeroth-order gradient updates. 
\end{remark}

\begin{remark}[\textbf{Dependence on $ \tau $}] \label{rem: dependence on tau}
It should be noted that the upper bound \eqref{eq: error bound in the main theorem}  
grows 
very slowly with the period of first-order updates $ \tau $, with an order of $ \mathcal{O} \left( 1 \right) $. 
As a result, 
for a fixed number of 
 iterations,  our algorithm needs 
 few rounds of first-order stochastic gradient vectors computation and  communication among the nodes to reach a certain error bound.   
This significantly contributes to the communication efficiency and computation efficiency of the proposed algorithm. 
Moreover, such growth rate significantly outperforms the existing results in the model averaging schemes, where the error upper bound grows quadratically or linearly with $ \tau $, which is a result of high model discrepancies due to local model updates \cite{yu2018parallel, zhou2017convergence, haddadpour2019trading}. 
\end{remark}

\section{Experiments}

In this section, we experimentally evaluate the performance of the proposed algorithm and compare it to various state-of-the-art  distributed 
algorithms, including the model averaging scheme  
RI-SGD \cite{haddadpour2019trading}, 
fully synchronous SGD 
(syncSGD) \cite{wang2018cooperative}, 
zeroth-order stochastic gradient method (ZO-SGD) \cite{sahu2019towards}, 
and the zeroth-order stochastic variance reduced gradient method 
(ZO-SVRG-Ave) \cite{liu2018zeroth}.
We evaluate the performance of the proposed algorithm on two different applications, as follows.   

\subsection{Generation of Adversarial Examples from DNNs}\label{sec: sim adversarial}

The first application is generation of adversarial examples from DNNs, which arises in testing the robustness of a de- ployed DNN to adversarial attacks. 
In the context of image classification, adversarial examples are carefully crafted perturbed images that are barely noticable and visually imperceptible, but when added to natural images, can fool the target model to mis-classify \cite{madry2017towards, liu2018zeroth}. 
In many applications dealing with mission-critical information, the robustness of a deployed DNN to adversarial attacks is highly critical for reliability of the model, e.g., traffic sign identification for autonomous driving. 
The task of generating a universal adversarial perturbation to $ K $ natural images can be regarded as an optimization problem of the form \eqref{eq: prob form 2}. 
More details on the problem formulation of  generating adversarial examples can be found in Appendix \ref{sec:App A}. 
Note that in general, the attacker can utilize the model evaluations and the parameters of the model to acquire its gradients \cite{madry2017towards}.

\textbf{Experimental Setup:} 
Similar to  \cite{liu2018zeroth},  
we apply the proposed algorithm and the baselines to generate adversarial examples to attack a 
well-trained DNN7 on the MNIST handwritten digit classification task, which achieves 99.4 test accuracy on natural examples. \footnote{\url{https://github.com/carlini/nn\_robust\_attacks}.} 
In our experiments, performed on a system with an Nvidia Tesla K80 GPU, we choose $ n = 10 $ examples from the same class, and set the 
 batch size and the number of workers to $ B = 5 $ and  $ m=5 $, respectively, for all the methods. We also use a constant step-size of $ 30/d $, where $ d=900 $  is the image dimension, and the smoothing parameter follows $ \mu = \mathcal{O}(1/\sqrt{dN}) $, where $ N $ is the number of iterations.

\textbf{Experimental Results:} 
Fig. \ref{fig: attack loss} depicts the attack loss versus the number of iterations, and Table \ref{tab: l2 distortion} shows the least $ l_2 $ distortion of the successful (universal) adversarial examples. 
It can be verified that compared to the zeroth-order methods of ZO-SGD and ZO-SVRG-Ave, the proposed method achieves significantly faster convergence and lower loss. Moreover, its convergence speed and attained loss is comparable to those of the fastest first-order methods, i.e., RI-SGD and Synchronous SGD. 
In terms of  $ l_2 $ distortion, 
the proposed method suggests better visual quality of the resulting adversarial examples than those in the previous zeroth-order methods, and its visual quality is similar to those of the first-order methods.

\begin{figure}[t]
\vskip 0.2in
\begin{center}
\centerline{\includegraphics[width=1
\columnwidth]
{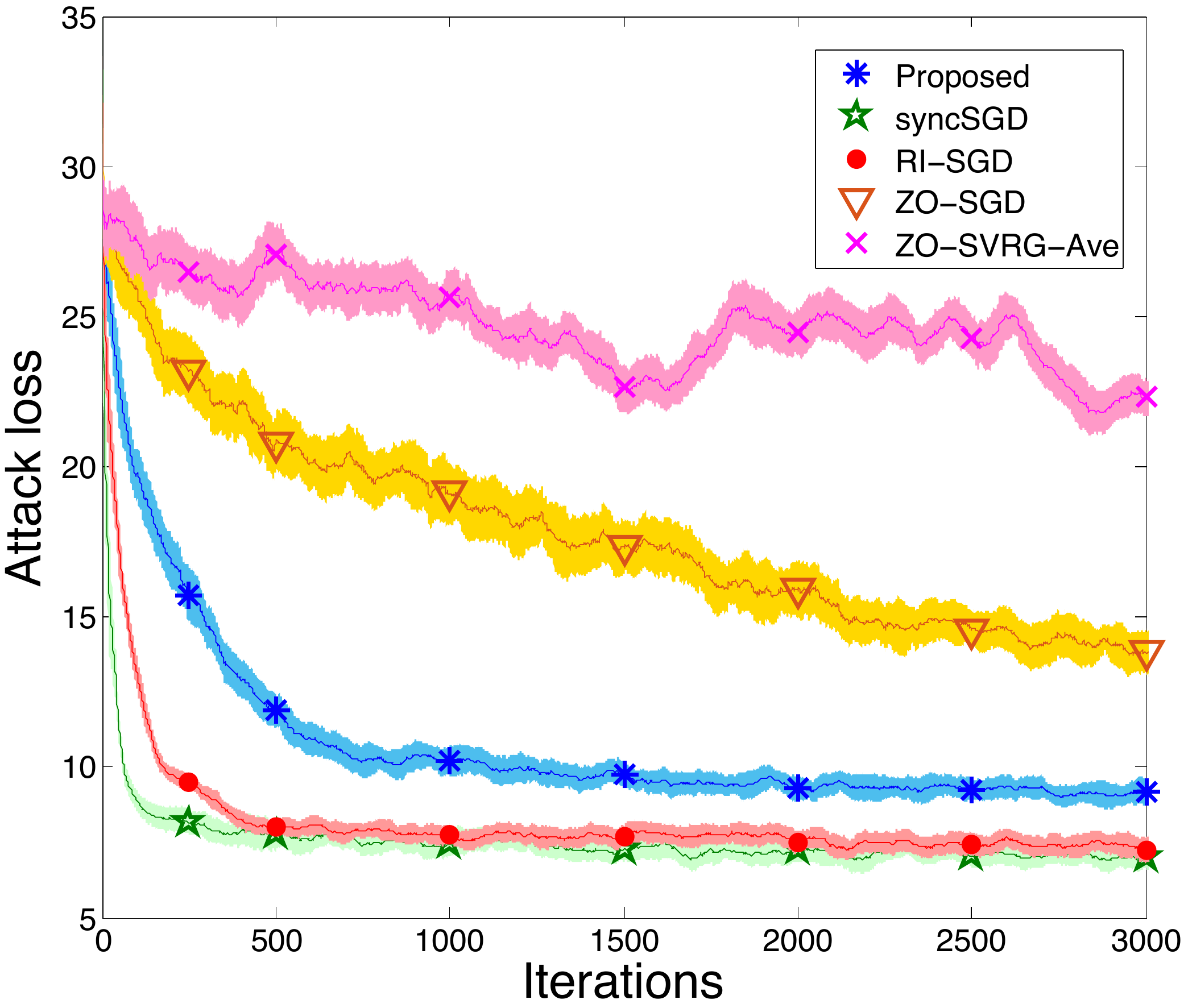}}
\caption{
Comparison of different methods for the task of generating universal adversarial examples from a DNN. Attack loss versus iterations.}
\label{fig: attack loss}
\end{center}
\vskip -0.2in
\end{figure}

\begin{table}[]
\caption{
$ l_2 $ distortion
}
\label{tab: l2 distortion}
\vskip 0.15in
\begin{center}
\begin{small}
\begin{sc}
\begin{tabular}{lcr}
\toprule
Method & $ l_2 $ distortion 
 \\
\midrule
RI-SGD      &  $6.08$  \\
syncSGD      &  $5.90$   \\
Proposed       &  $8.86$   \\
ZO-SGD      &  $10.07$  \\
ZO-SVRG-Ave      &   $ 16.41 $  \\  
\bottomrule
\end{tabular}
\end{sc}
\end{small}
\end{center}
\vskip -0.1in
\end{table}

Table \ref{attacked images} shows the original natural images along with the resulted adversarial examples generated by different methods, 
 under the elaborated experimental setup. 

\begin{table*}[bh]
\caption{The generated adversarial examples from a well-trained DNN7 on MNIST using the proposed method and the baselines.}
\label{attacked images}
\vskip 0.15in
\begin{center}
\begin{small}
\begin{sc}
\begin{tabular}{lcccccccccc}
\toprule
Image ID & 6 & 72 & 128 & 211 & 315 & 398 &  475 &  552 &  637 & 738 \\
\midrule
Original & \begin{minipage}{.05\textwidth}
      \includegraphics[width=\linewidth]{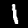}
    \end{minipage}
     & \begin{minipage}{.05\textwidth}
      \includegraphics[width=\linewidth]{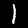}
    \end{minipage}
     & \begin{minipage}{.05\textwidth}
      \includegraphics[width=\linewidth]{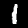}
    \end{minipage}
     & \begin{minipage}{.05\textwidth}
      \includegraphics[width=\linewidth]{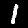}
    \end{minipage}
     & \begin{minipage}{.05\textwidth}
      \includegraphics[width=\linewidth]{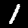}
    \end{minipage}
     & \begin{minipage}{.05\textwidth}
      \includegraphics[width=\linewidth]{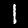}
    \end{minipage} 
     &  \begin{minipage}{.05\textwidth}
      \includegraphics[width=\linewidth]{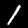}
    \end{minipage}
     &  \begin{minipage}{.05\textwidth}
      \includegraphics[width=\linewidth]{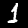}
    \end{minipage}
     &  \begin{minipage}{.05\textwidth}
      \includegraphics[width=\linewidth]{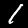}
    \end{minipage}
     & \begin{minipage}{.05\textwidth}
      \includegraphics[width=\linewidth]{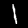}
    \end{minipage}
     \\
\midrule
HO-SGD (Proposed)    & \begin{minipage}{.05\textwidth}
      \includegraphics[width=\linewidth]{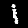}
    \end{minipage}
     & \begin{minipage}{.05\textwidth}
      \includegraphics[width=\linewidth]{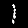}
    \end{minipage}
     & \begin{minipage}{.05\textwidth}
      \includegraphics[width=\linewidth]{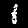}
    \end{minipage}
     & \begin{minipage}{.05\textwidth}
      \includegraphics[width=\linewidth]{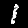}
    \end{minipage}
     & \begin{minipage}{.05\textwidth}
      \includegraphics[width=\linewidth]{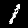}
    \end{minipage}
     & \begin{minipage}{.05\textwidth}
      \includegraphics[width=\linewidth]{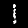}
    \end{minipage} 
     &  \begin{minipage}{.05\textwidth}
      \includegraphics[width=\linewidth]{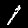}
    \end{minipage}
     &  \begin{minipage}{.05\textwidth}
      \includegraphics[width=\linewidth]{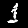}
    \end{minipage}
     &  \begin{minipage}{.05\textwidth}
      \includegraphics[width=\linewidth]{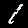}
    \end{minipage}
     & \begin{minipage}{.05\textwidth}
      \includegraphics[width=\linewidth]{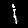}
    \end{minipage} \\
Classified as    & 3 & 3 & 8 & 8 & 8 & 3 &  7 &  3 &  8 & 3 \\
\midrule
ZO-SGD    & \begin{minipage}{.05\textwidth}
      \includegraphics[width=\linewidth]{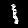}
    \end{minipage}
     & \begin{minipage}{.05\textwidth}
      \includegraphics[width=\linewidth]{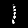}
    \end{minipage}
     & \begin{minipage}{.05\textwidth}
      \includegraphics[width=\linewidth]{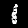}
    \end{minipage}
     & \begin{minipage}{.05\textwidth}
      \includegraphics[width=\linewidth]{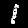}
    \end{minipage}
     & \begin{minipage}{.05\textwidth}
      \includegraphics[width=\linewidth]{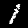}
    \end{minipage}
     & \begin{minipage}{.05\textwidth}
      \includegraphics[width=\linewidth]{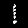}
    \end{minipage} 
     &  \begin{minipage}{.05\textwidth}
      \includegraphics[width=\linewidth]{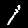}
    \end{minipage}
     &  \begin{minipage}{.05\textwidth}
      \includegraphics[width=\linewidth]{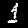}
    \end{minipage}
     &  \begin{minipage}{.05\textwidth}
      \includegraphics[width=\linewidth]{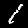}
    \end{minipage}
     & \begin{minipage}{.05\textwidth}
      \includegraphics[width=\linewidth]{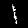}
    \end{minipage} \\
Classified as    & 3 & 8 & 8 & 8 & 8 & 8 &  7 &  3 &  8 & 3 \\
\midrule
ZO-SVRG-Ave 
& \begin{minipage}{.05\textwidth}
      \includegraphics[width=\linewidth]{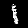}
    \end{minipage}
     & \begin{minipage}{.05\textwidth}
      \includegraphics[width=\linewidth]{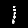}
    \end{minipage}
     & \begin{minipage}{.05\textwidth}
      \includegraphics[width=\linewidth]{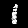}
    \end{minipage}
     & \begin{minipage}{.05\textwidth}
      \includegraphics[width=\linewidth]{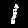}
    \end{minipage}
     & \begin{minipage}{.05\textwidth}
      \includegraphics[width=\linewidth]{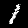}
    \end{minipage}
     & \begin{minipage}{.05\textwidth}
      \includegraphics[width=\linewidth]{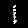}
    \end{minipage} 
     &  \begin{minipage}{.05\textwidth}
      \includegraphics[width=\linewidth]{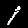}
    \end{minipage}
     &  \begin{minipage}{.05\textwidth}
      \includegraphics[width=\linewidth]{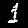}
    \end{minipage}
     &  \begin{minipage}{.05\textwidth}
      \includegraphics[width=\linewidth]{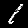}
    \end{minipage}
     & \begin{minipage}{.05\textwidth}
      \includegraphics[width=\linewidth]{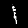}
    \end{minipage} \\
Classified as    & 8 & 7 & 8 & 8 & 8 & 8 &  7 &  3 &  8 & 3 \\
\midrule 
syncSGD    & \begin{minipage}{.05\textwidth}
      \includegraphics[width=\linewidth]{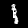}
    \end{minipage}
     & \begin{minipage}{.05\textwidth}
      \includegraphics[width=\linewidth]{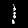}
    \end{minipage}
     & \begin{minipage}{.05\textwidth}
      \includegraphics[width=\linewidth]{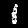}
    \end{minipage}
     & \begin{minipage}{.05\textwidth}
      \includegraphics[width=\linewidth]{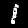}
    \end{minipage}
     & \begin{minipage}{.05\textwidth}
      \includegraphics[width=\linewidth]{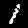}
    \end{minipage}
     & \begin{minipage}{.05\textwidth}
      \includegraphics[width=\linewidth]{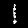}
    \end{minipage} 
     &  \begin{minipage}{.05\textwidth}
      \includegraphics[width=\linewidth]{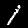}
    \end{minipage}
     &  \begin{minipage}{.05\textwidth}
      \includegraphics[width=\linewidth]{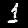}
    \end{minipage}
     &  \begin{minipage}{.05\textwidth}
      \includegraphics[width=\linewidth]{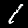}
    \end{minipage}
     & \begin{minipage}{.05\textwidth}
      \includegraphics[width=\linewidth]{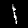}
    \end{minipage} \\       
Classified as    & 3 & 8 & 8 & 8 & 8 & 8 &  7 &  3 &  8 & 3\\
\midrule
RI-SGD    & \begin{minipage}{.05\textwidth}
      \includegraphics[width=\linewidth]{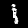}
    \end{minipage}
     & \begin{minipage}{.05\textwidth}
      \includegraphics[width=\linewidth]{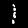}
    \end{minipage}
     & \begin{minipage}{.05\textwidth}
      \includegraphics[width=\linewidth]{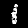}
    \end{minipage}
     & \begin{minipage}{.05\textwidth}
      \includegraphics[width=\linewidth]{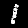}
    \end{minipage}
     & \begin{minipage}{.05\textwidth}
      \includegraphics[width=\linewidth]{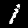}
    \end{minipage}
     & \begin{minipage}{.05\textwidth}
      \includegraphics[width=\linewidth]{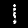}
    \end{minipage} 
     &  \begin{minipage}{.05\textwidth}
      \includegraphics[width=\linewidth]{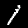}
    \end{minipage}
     &  \begin{minipage}{.05\textwidth}
      \includegraphics[width=\linewidth]{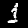}
    \end{minipage}
     &  \begin{minipage}{.05\textwidth}
      \includegraphics[width=\linewidth]{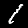}
    \end{minipage}
     & \begin{minipage}{.05\textwidth}
      \includegraphics[width=\linewidth]{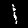}
    \end{minipage} \\       
Classified as    & 3 & 3 & 8 & 8 & 8 & 3 &  7 &  3 &  8 & 3\\
\bottomrule
\end{tabular}
\end{sc}
\end{small}
\end{center}
\vskip -0.1in
\end{table*}

\begin{table*}[]
\caption{Details of the datasets used in the second experiments.}
\label{datasets}
\vskip 0.15in
\begin{center}
\begin{small}
\begin{sc}
\begin{tabular}{lccccp{4.6 cm}} 
\toprule
Dataset & $ \verb|#| $ Classes & $ \verb|#| $ Training data  & $ \verb|#| $ Testing data & $ \verb|#| $ Features & Description  \\
\midrule
SENSORLESS    & $ 11 $ & $ 48509 $  &  $ 10000 $ &  $ 48 $ & 
Sensor-less drive diagnosis \cite{wang2018distributed} \\
\midrule
ACOUSTIC    & $ 3 $ & $ 78823 $ & $19705$ &  $ 50 $ & Accoustic vehicle classification in distributed sensor networks \cite{duarte2004vehicle} \\
\midrule
COVTYPE    & $ 7 $ & $ 50000 $  &  $ 81012 $ &  $ 54 $ & Forest cover type prediction from cartographic variables only \cite{asuncion2007uci}\\
\midrule
SEISMIC    & $ 3 $ & $ 78823 $  &  $ 19705 $ &  $ 50 $ & Seismic vehicle classification in distributed sensor networks \cite{duarte2004vehicle} \\
\bottomrule
\end{tabular}
\end{sc}
\end{small}
\end{center}
\vskip -0.1in
\end{table*}

\subsection{Multi-Class Classification Tasks with Multiple Workers}




The second experiment includes various multi-class classification tasks performed by multiple worker nodes in a distributed environment. We use four  different famous datasets, including COVTYPE, 
SensIT Vehicle (both ACOUSTIC and SEISMIC),  
and SENSORLESS. 
{\footnote{All the datasets 
are available online at \url{https://www.csie.ntu.edu.tw/\~ cjlin/libsvmtools/datasets/multiclass.html}. }} Details and the description of each dataset can be found in Table \ref{datasets}.  

\textbf{Experimental Setup:} As the base model for training, we choose a high-dimensional fully connected two-layer neural network with more than $ 1.69 M $ parameters ($ 1.3 K $ and $ 1.3 K $ hidden neurons, consecutively), i.e., $ d> 1.69 \times 10^6 $. 
We use PyTorch \cite{paszke2019pytorch} to develop the proposed algorithm and the baselines in a distributed environment, and conduct different types of experiments on 
a system with 8 Cores of CPU and 4 Nvidia Tesla K80 GPUs. 

The number of worker nodes and the batch size are set to $ m=4 $ and $ B= 64 $, respectively, for all the methods, and the period of the first-order gradient exchanges is set to $ \tau = 8 $ for the proposed method and the periodic model averaging method of RI-SGD. Moreover, a redundancy factor of $ \mu= 0.25 $ is considered for the RI-SGD method. All the methods are run from the same initial points.  
Finally, it should be noted that for each dataset, we have optimized the learning rates of all the methods, in order to 
have a fair comparison.


\textbf{Experimental Results:} The performance of different methods in distributed training of the considered model for various datasets is compared in Fig. \ref{fig: experiment2_all_datasets}. It shows the training loss versus iterations, the training loss versus wall-clock time (in seconds), and the testing accuracy versus wall-clock time, for different  datasets as aforementioned. As can be verified from the figures, the proposed method significantly outperforms the zeroth-order method ZO-SGD, in terms of convergence speed, wall-clock time, and testing accuracy. 
Moreover, 
despite the high dimension of the model, the performance of the proposed method is still comparable to the first-order methods of sync-SGD and RI-SGD (which is a model averaging method with the complementary help of infused redundancy), while the proposed method benefits from lower computational complexity as discussed before. 
These experimental findings comply with our theoretical results discussed in Section \ref{sec: convergence analysis}. 

\begin{figure*}[t]
	\begin{minipage}[c][1\width]{
	   0.3\textwidth}
	   \centering
	   \includegraphics[width=1.1\textwidth]{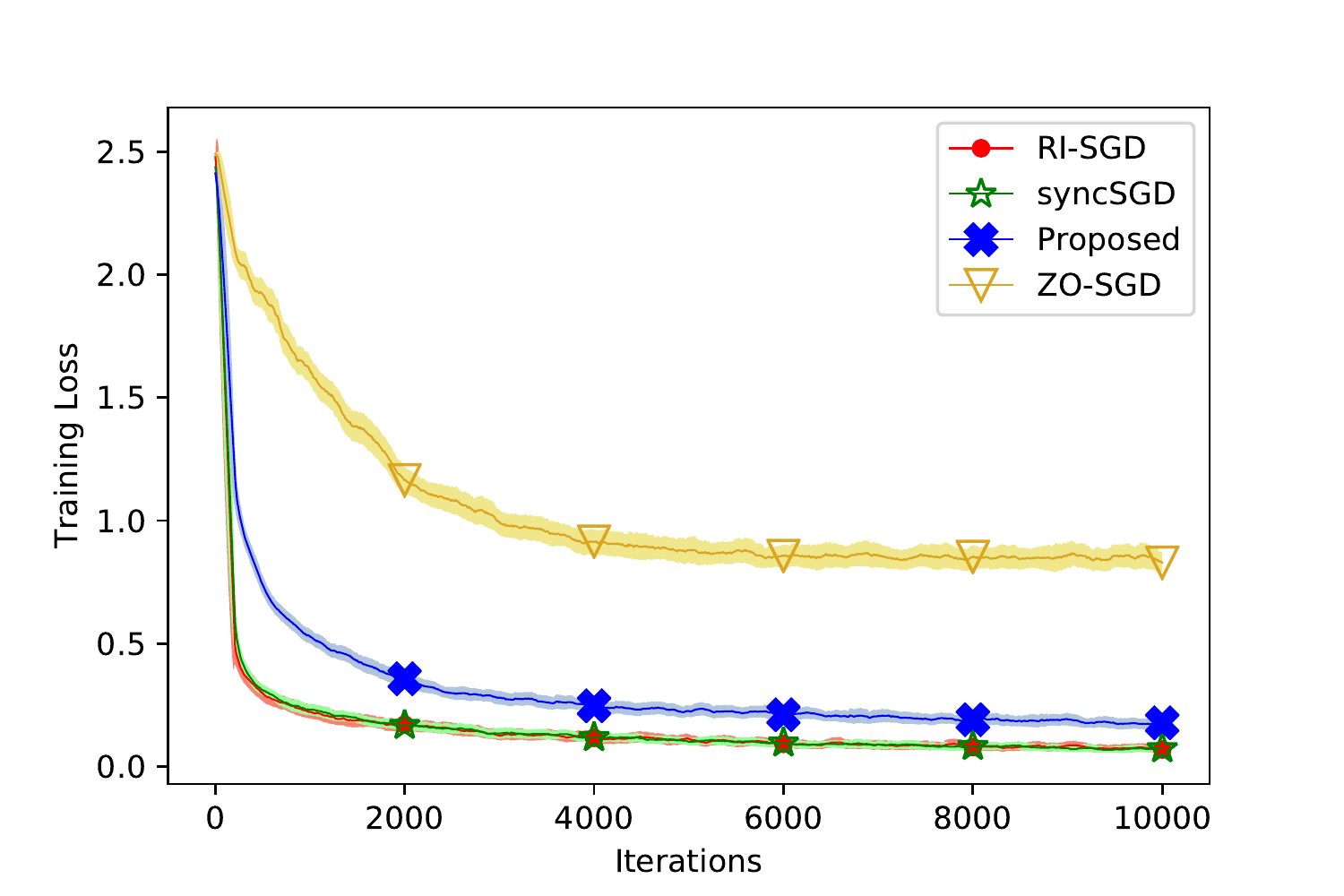}
	\end{minipage}
 \hfill 	
	\begin{minipage}[c][1\width]{
	   0.3\textwidth}
	   \centering
	   \includegraphics[width=1.1\textwidth]{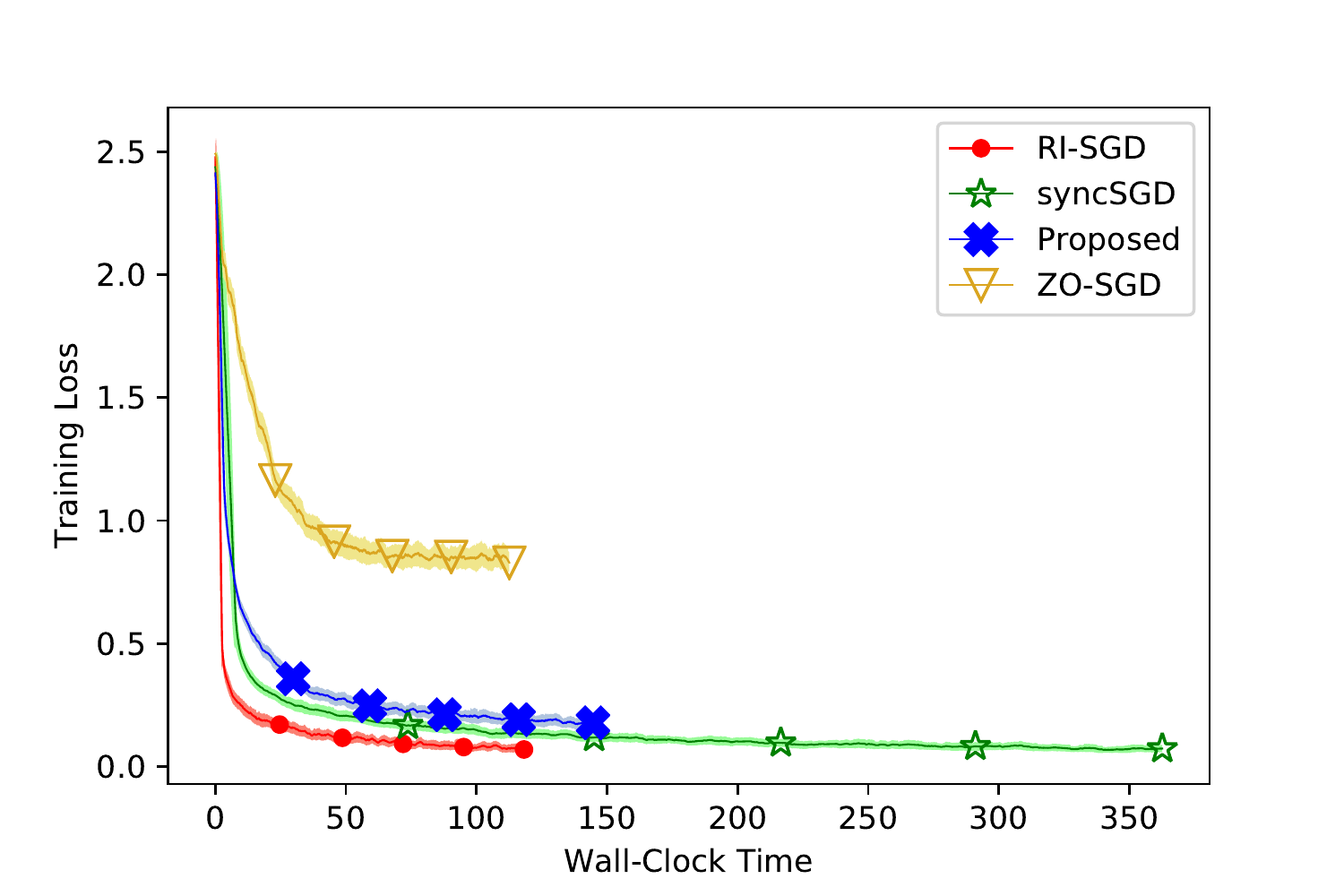}
	\end{minipage}
 \hfill	
	\begin{minipage}[c][1\width]{
	   0.3\textwidth}
	   \centering
	   \includegraphics[width=1.1\textwidth]{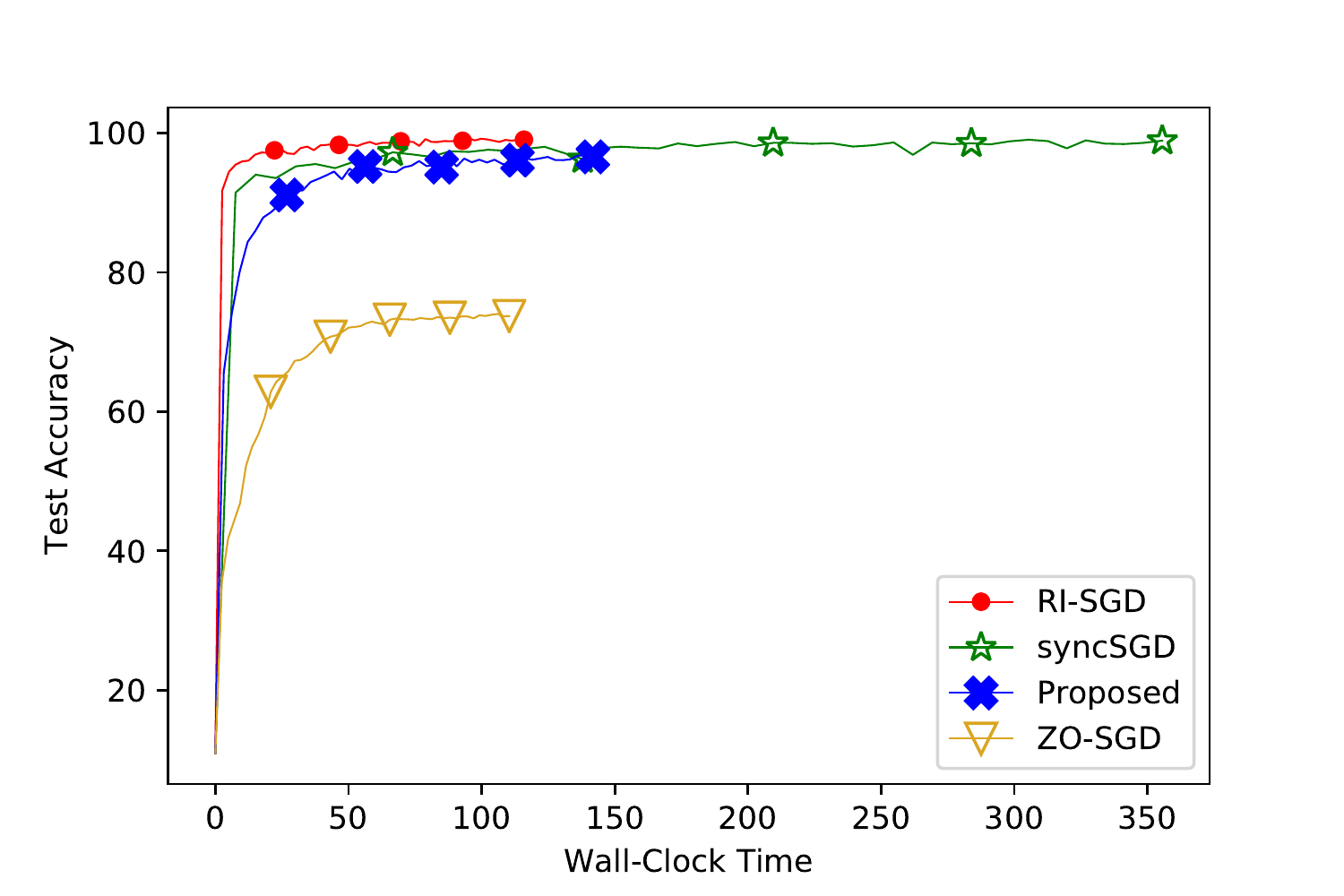}
	\end{minipage}
		\newline
	\begin{minipage}[c][1\width]{
	   0.3\textwidth}
	   \centering
	   \includegraphics[width=1.1\textwidth]{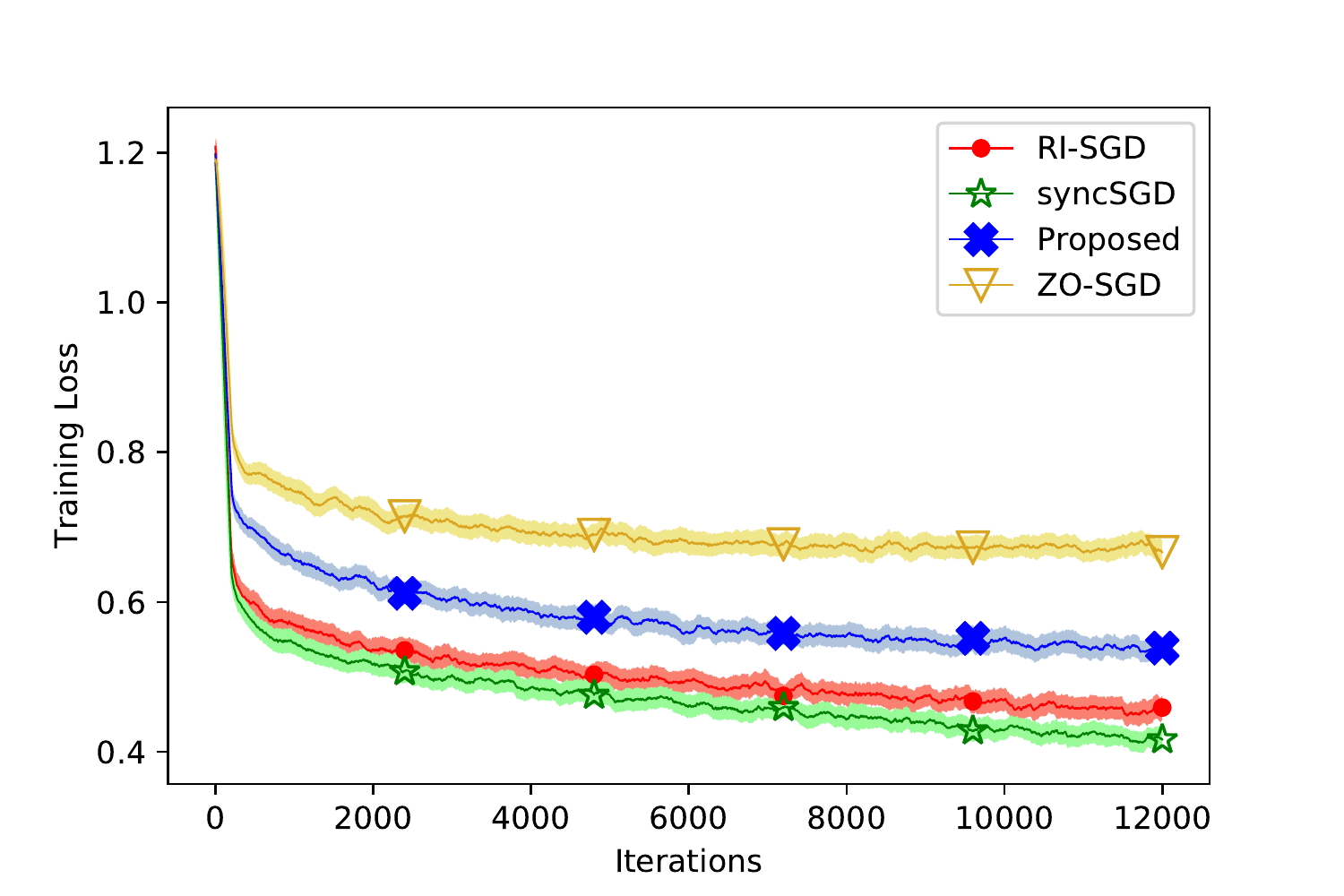}
	\end{minipage}
 \hfill 	
	\begin{minipage}[c][1\width]{
	   0.3\textwidth}
	   \centering
	   \includegraphics[width=1.1\textwidth]{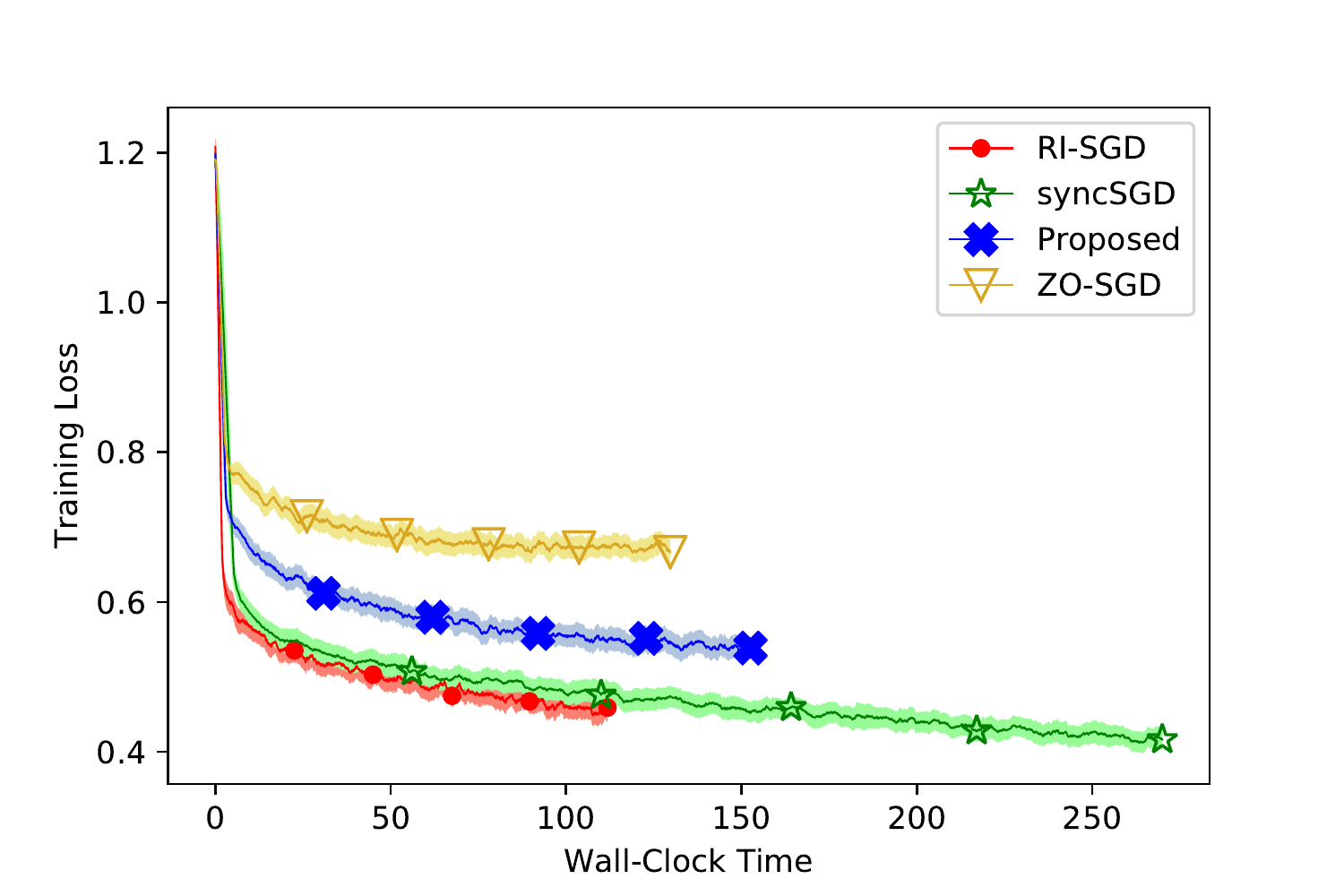}
	\end{minipage}
 \hfill	
	\begin{minipage}[c][1\width]{
	   0.3\textwidth}
	   \centering
	   \includegraphics[width=1.1\textwidth]{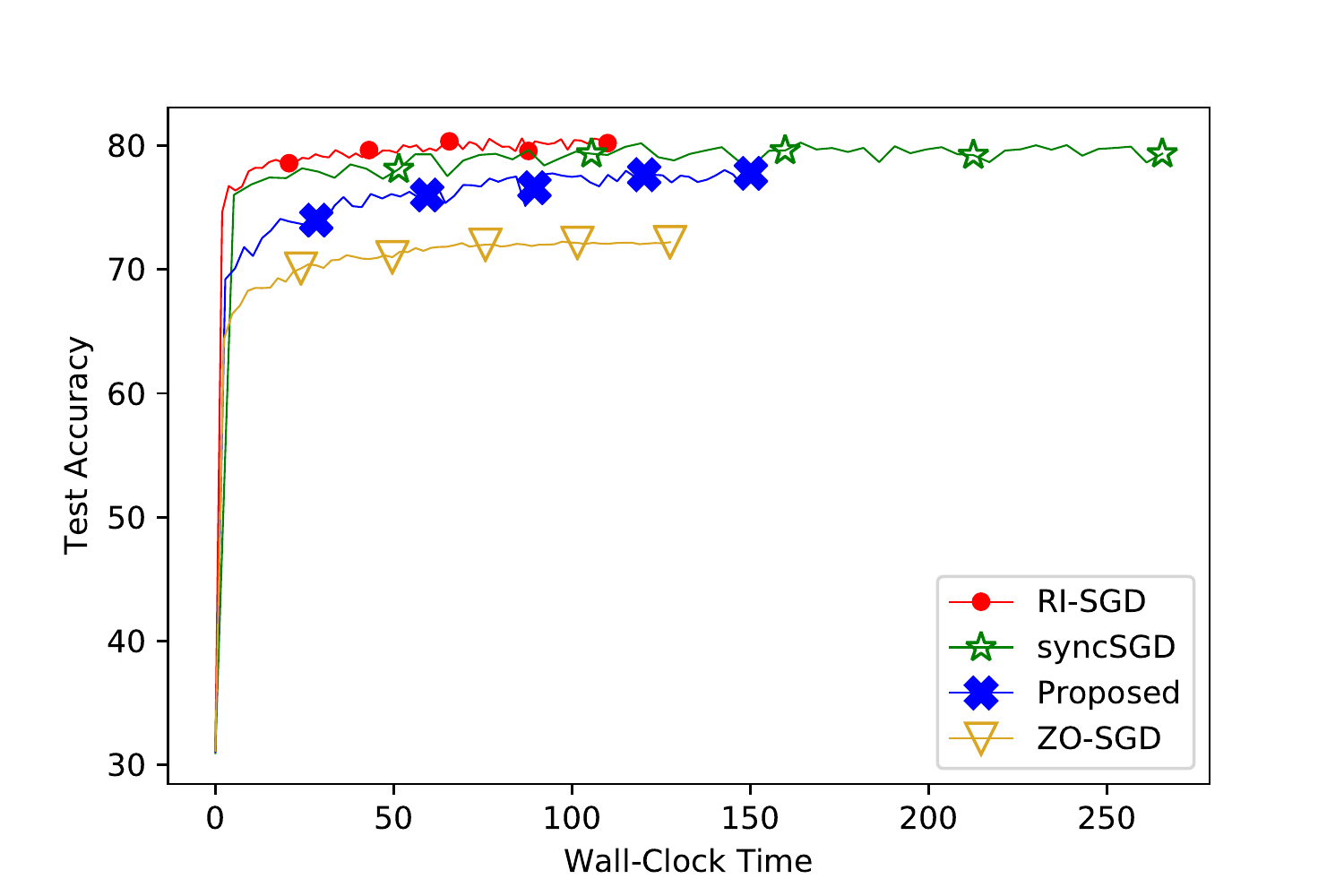}
	\end{minipage}
		\newline
	\begin{minipage}[c][1\width]{
	   0.3\textwidth}
	   \centering
	   \includegraphics[width=1.1\textwidth]{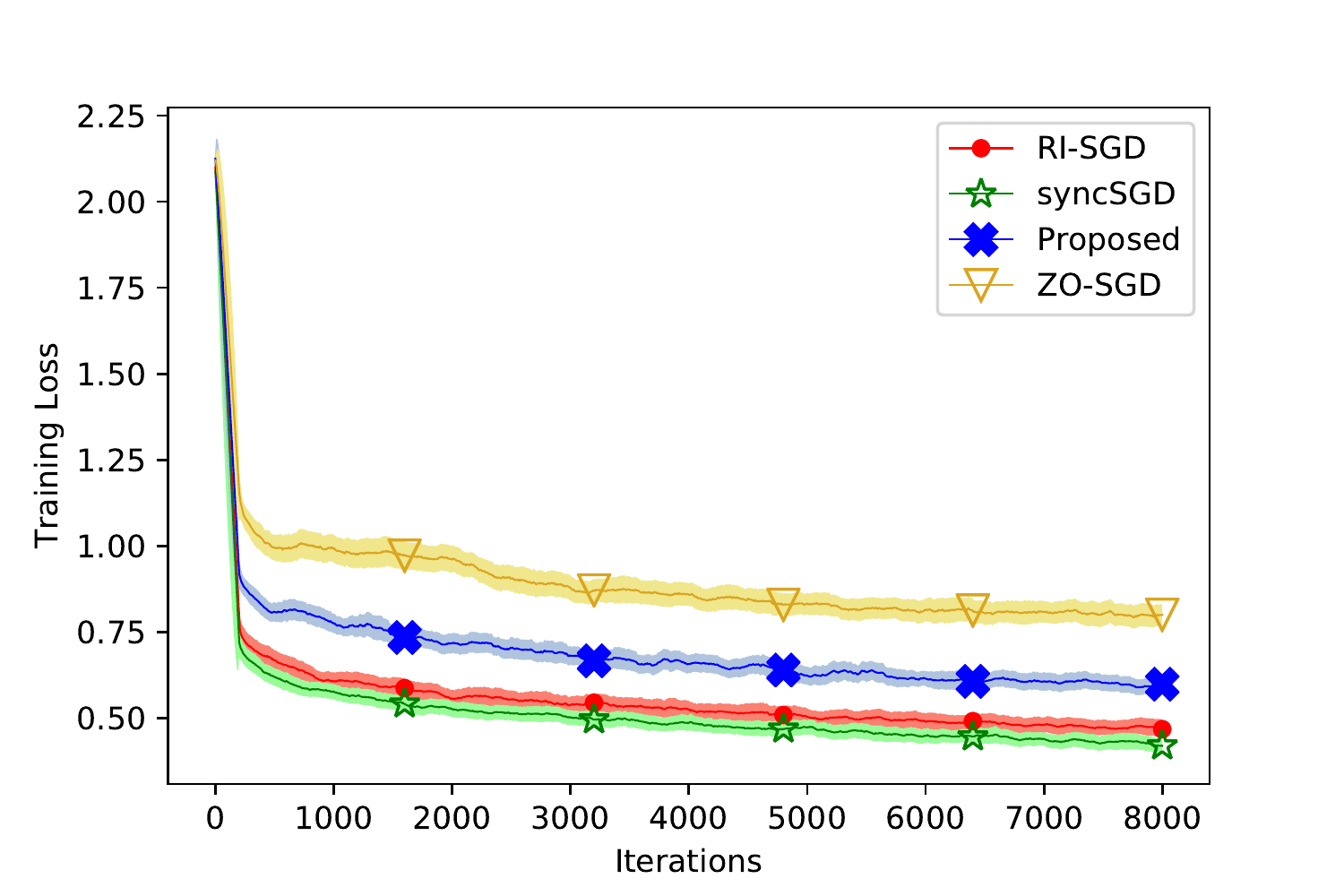}
	\end{minipage}
 \hfill 	
	\begin{minipage}[c][1\width]{
	   0.3\textwidth}
	   \centering
	   \includegraphics[width=1.1\textwidth]{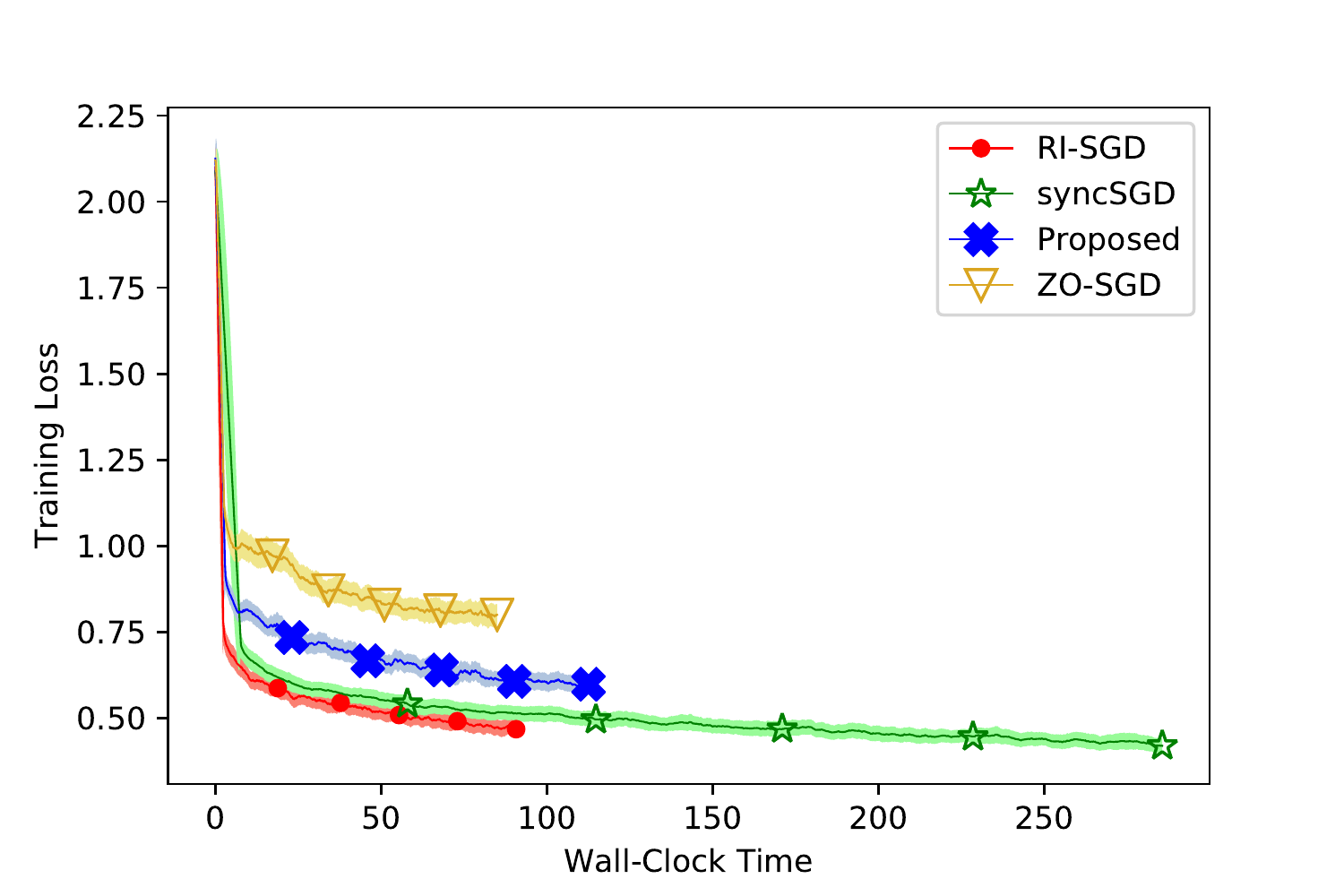}
	\end{minipage}
 \hfill	
	\begin{minipage}[c][1\width]{
	   0.3\textwidth}
	   \centering
	   \includegraphics[width=1.1\textwidth]{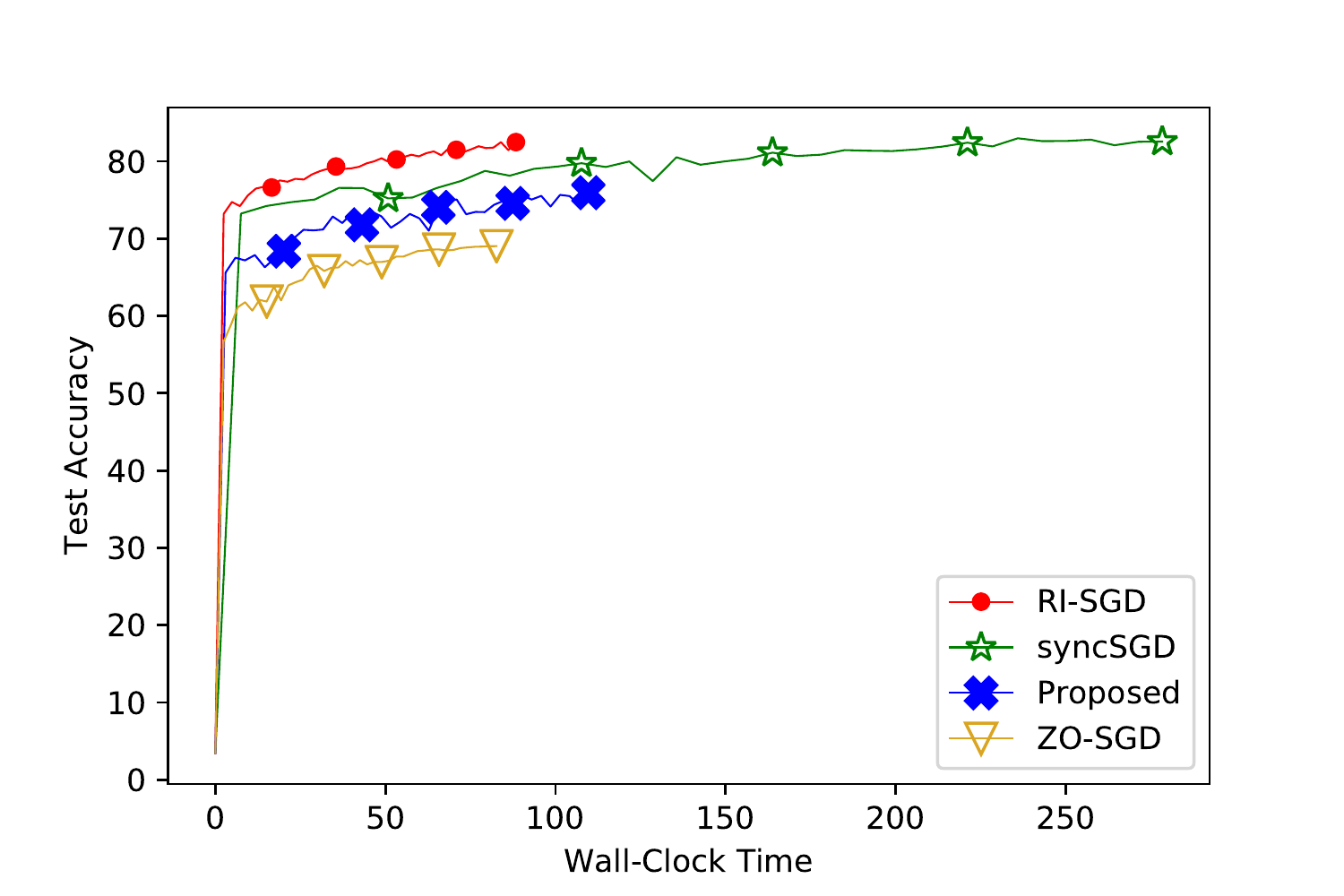}
	\end{minipage}
	\newline
	\begin{minipage}[c][1\width]{
	   0.3\textwidth}
	   \centering
	   \includegraphics[width=1.1\textwidth]{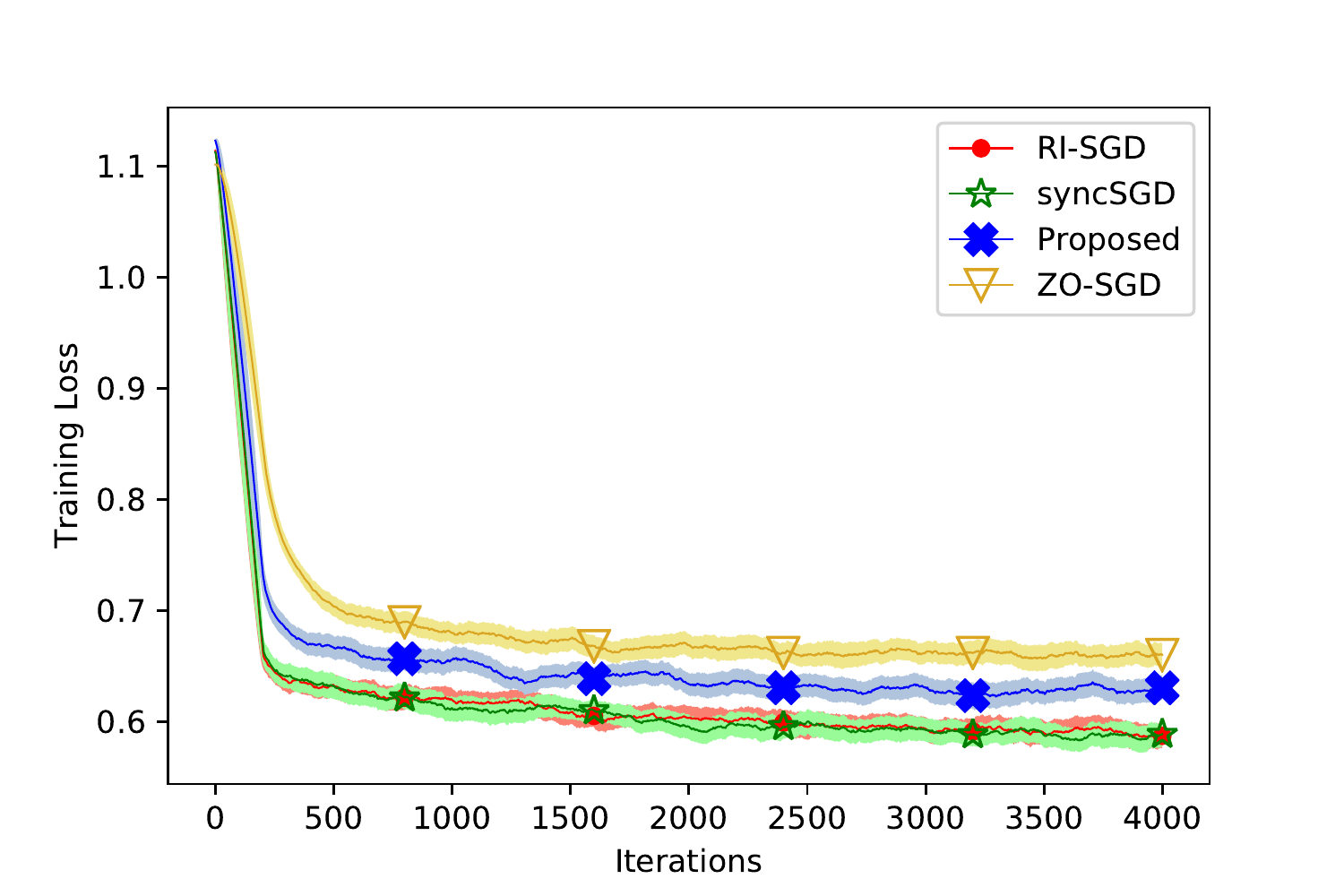}
	\end{minipage}
 \hfill 	
	\begin{minipage}[c][1\width]{
	   0.3\textwidth}
	   \centering
	   \includegraphics[width=1.1\textwidth]{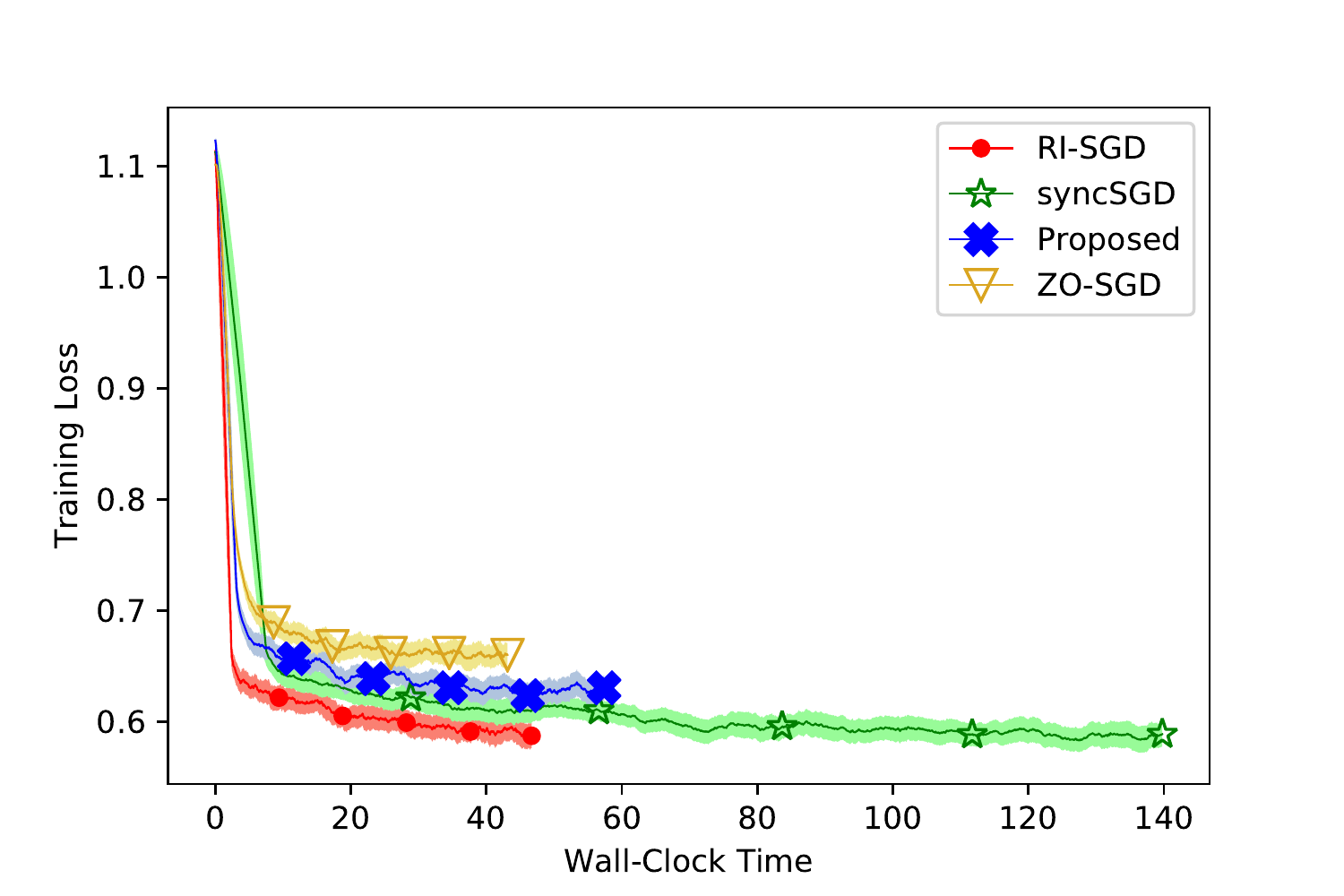}
	\end{minipage}
 \hfill	
	\begin{minipage}[c][1\width]{
	   0.3\textwidth}
	   \centering
	   \includegraphics[width=1.1\textwidth]{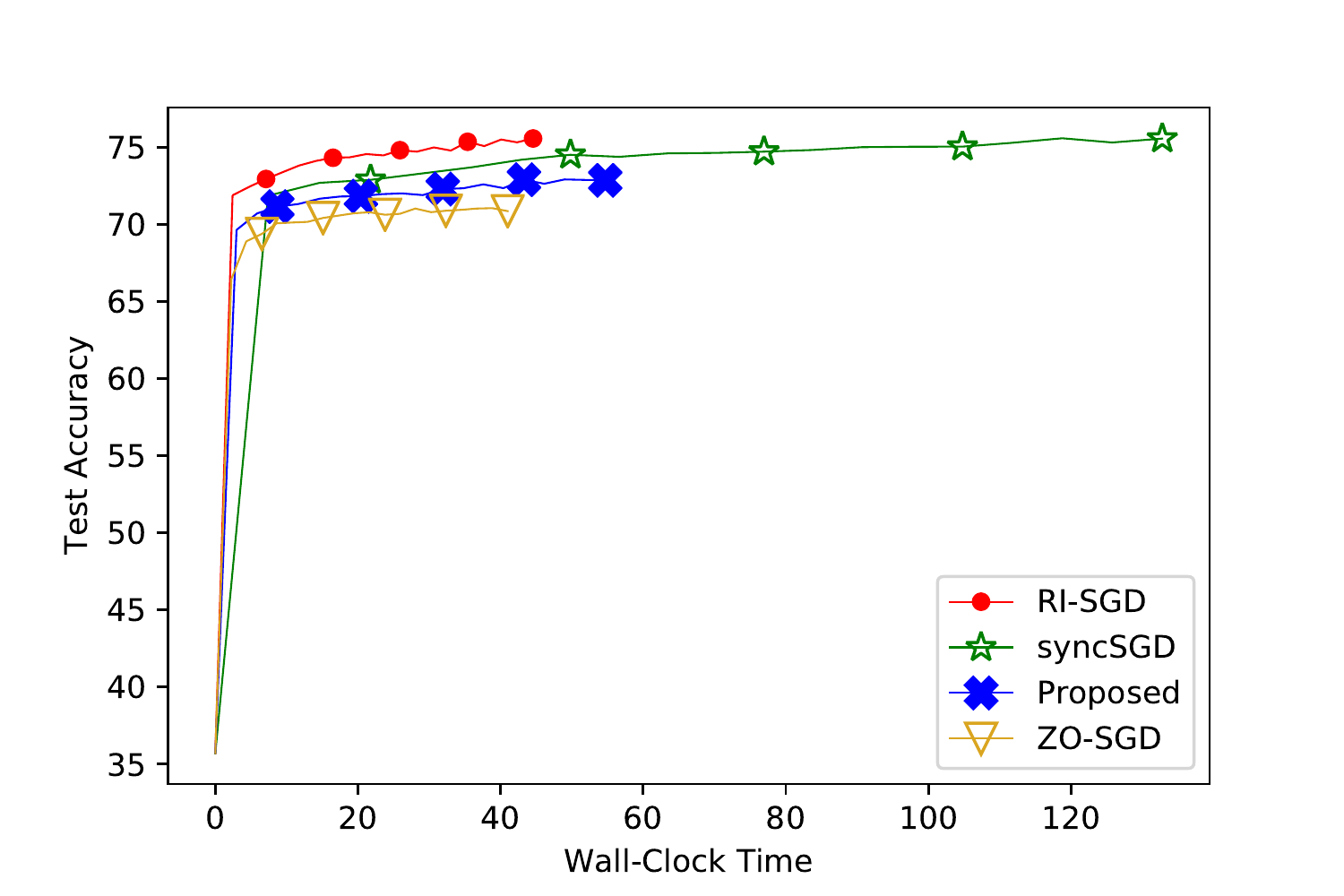}
	\end{minipage}
\caption{Training loss and test accuracy comparisons of different methods. The rows correspond to SENSORLESS, ACOUSTIC, COVTYPE, and SEISMIC datasets, from top to down, respectively. Wall-clock times are measured in seconds.}\label{fig: experiment2_all_datasets}
\end{figure*}

\clearpage

\section{Conclusion}

In this paper, we proposed a hybrid zeroth/first-order 
distributed SGD method for solving non-convex unconstrained stochastic optimization problems 
that benefits from low communication overhead and computational complexity, and yet converges fast. 
By theoretical analyses, we showed that the proposed algorithm reaches the same iteration complexity as the first-order distributed SGD algorithms, while it enjoys order-wisely lower computational complexity and comparable communication overhead.   Moreover, the proposed algorithm significantly outperforms the convergence rates of the existing zeroth-order algorithms, guaranteeing a convergence rate of the order of $ \mathcal{O} \left( d / \sqrt{m N} \right) $, with a comparable computational complexity. 
Experimental results 
demonstrate the effectiveness of the proposed approach compared to various state-of-the-art methods.


\bibliography{references}
\bibliographystyle{icml2020}


\newpage
\onecolumn






\begin{appendices}

\gdef\thesection{Appendix \Alph{section}}

\gdef\thesubsection{\Alph{section}.\arabic{subsection}}

\section{Problem Formulation of Generating Adversarial Examples from DNN}\label{sec:App A}


Here, we elaborate the details of  
the problem formulation for the  task of 
generating adversarial examples considered 
in Section \ref{sec: sim adversarial}.  
The task of generating a universal adversarial perturbation to $ K $ natural images can be regarded as an optimization problem of the form \eqref{eq: prob form 2}, in which $ F \left( \boldsymbol{x} , \boldsymbol{\zeta}_k \right) $ is the designed attack loss function of the $ k^\text{th} $ image, and is defined as \cite{carlini2017towards}  
\begin{align}
F \left( \boldsymbol{x} , \boldsymbol{\zeta}_k \right) = & c. \max \left\lbrace 0, f_{y_k} \left(  0.5 \tanh \left( \tanh^{-1} 2 \boldsymbol{a}_k + \boldsymbol{x} \right) \right) - \max_{j \neq y_j} f_{j} \left(  0.5 \tanh \left( \tanh^{-1} 2 \boldsymbol{a}_k + \boldsymbol{x} \right) \right)  \right\rbrace \notag \\
&+ \parallel 0.5 \tanh \left( \tanh^{-1} 2 \boldsymbol{a}_k + \boldsymbol{x} \right) - \boldsymbol{a}_k  \parallel^2_2, \notag
\end{align}
where $ \boldsymbol{\zeta}_k = \left( \boldsymbol{a}_k, y_k \right) $ in which, $ \boldsymbol{a}_k \in \left[ -0.5,0.5 \right]^d $ and $ y_k  $ denote 
the $ k^\text{th} $ natural image and its original class label, respectively. Furthermore, for a generated adversarial example $ \boldsymbol{z} $ and for each image class $ i=1,\cdots,I $, the function 
$  f_i \left( \boldsymbol{z} \right)  $ outputs prediction score of the model (e.g., log-probability) in that class. Note that in the above formulation, the $ \tanh $ operation is used to keep the generated adversarial example, i.e., $  \boldsymbol{z} = 0.5 \tanh \left( \tanh^{-1} 2 \boldsymbol{a}_k + \boldsymbol{x} \right) $, in the valid image space $ \left[ -0.5,0.5 \right]^d $. 
Finally,  $ c $ is a regularization parameter to balance a trade-off between the adversarial attack success and the $ l_2 $ distortion of the adversarial example. Note that lower values for $ c $ result in lower $ l_2 $ distortion which suggests better visual quality of the resulting adversarial examples, while higher values for $ c $ results in higher success rate of the adversarial attack (i.e., it can perturb more natural images to be misclassified).

\section{Convergence Proofs}\label{sec: app proofs} 



In order to prove 
 Theorem \ref{th: main}, we first define some auxiliary variables and functions, then we present and prove some useful lemmas which will be utilized for proving Theorem \ref{th: main}. Note that unless stated otherwise, the expectations are with respect to all the random variables.

\subsection{Preliminary Definitions}

We use the following smoothing scheme 
to approximate the first-order information of a given  function $ f $. 
\begin{definition}\label{def: f_mu}
For any smoothing parameter $ \mu > 0 $, the co-called smoothing function of any original function $ f \left( \boldsymbol{x} \right) $ is defined as \cite{darzentas1984problem}:  
\begin{equation}
f_\mu \left( \boldsymbol{x} \right) \triangleq \mathbb{E}_{\boldsymbol{u} \sim U_b} \left[ f \left( \boldsymbol{x} + \mu \boldsymbol{u} \right) \right],
\end{equation}
where 
$ U_b $ is the uniform distribution over the Euclidean $ d $-dimensional  unit  ball. 
\end{definition}

The above smoothing scheme 
exhibits several interesting features 
 that we shall use for the analysis of the convergence of the proposed method. 
 For example, as shown by Lemma 4.1(b) in \cite{gao2018information},  
for any function $ f \in C_L^1 \left( \mathbb{R}^d \right) $ (this condition is satisfied under \eqref{eq: Lips grad} in Assumption \ref{assump: Lipschitz grad + bounded grad}), we have
\begin{equation}\label{eq: f_mu - f}
| f_\mu \left( \boldsymbol{x} \right) - f \left( \boldsymbol{x} \right) | \leq \dfrac{\mu^2 L}{2},
\end{equation}
and
\begin{equation}\label{eq: grad f_mu - grad f}
\parallel \nabla f_\mu \left( \boldsymbol{x} \right) - \nabla f \left( \boldsymbol{x} \right) \parallel \leq \dfrac{\mu L d}{2}.
\end{equation}


\begin{definition}\label{def: tilde_f_G_delta}
Here, we define three functions that will be extensively used in the analysis of the proposed algorithm.

\begin{itemize}

\item[(a)] The 
function $ \tilde{f}_t \left( \boldsymbol{x} \right) $ indicates the 
considered function at each iteration of the proposed algorithm (that is being minimized by using its stochastic gradient direction): 
\begin{equation}\label{eq: tilde_f} 
\tilde{f}_t \left( \boldsymbol{x} \right) \triangleq \left\{ \begin{array}{ll}
{f \left( \boldsymbol{x} \right) }&{\mathrm{If} \mod\left(t,\tau\right)=0, }\\
\\
{f_\mu \left( \boldsymbol{x} \right)}&{\mathrm{Otherwise.}}
\end{array} 
\right.
\end{equation}

\item[(b)] Moreover, $ \tilde{\boldsymbol{G}}_t $ indicates the update direction, or equivalently, the gradient estimation of the objective function, utilized at each iteration of the proposed algorithm: 
\begin{equation}\label{eq: tilde_G} 
\tilde{\boldsymbol{G}}_t \left( \boldsymbol{x} \right) \triangleq \left\{ \begin{array}{ll}
{\dfrac{1}{Bm} \sum_{i=1}^m \sum_{b=1}^B \nabla F \left( \boldsymbol{x} , \boldsymbol{\zeta}_{t+1,i,b} \right) }&{\mathrm{If} \mod\left(t,\tau\right)=0, }\\
\\
{\dfrac{1}{Bm} \sum_{i=1}^m \sum_{b=1}^B \boldsymbol{G}_\mu \left( \boldsymbol{x} , \boldsymbol{\zeta}_{t+1,i,b} , \boldsymbol{v}_{t+1,i} \right) }&{\mathrm{Otherwise,}}
\end{array} 
\right.
\end{equation} 
where $ \boldsymbol{G}_\mu \left( \boldsymbol{x} , \boldsymbol{\zeta}_{t+1,i,b} , \boldsymbol{v}_{t+1,i} \right) $ is the zeroth-order gradient estimation and is defined as 
\begin{align}
\boldsymbol{G}_\mu \big(  \boldsymbol{x} , \boldsymbol{\zeta}_{t+1,i,b} , \boldsymbol{v}_{t+1,i} \big) \triangleq  \dfrac{d}{\mu} \Big[ F\left( \boldsymbol{x}^t + \mu \boldsymbol{v}_{t+1,i} , \boldsymbol{\zeta}_{t+1,i,b} \right) - F\left( \boldsymbol{x}^t , \boldsymbol{\zeta}_{t+1,i,b} \right) \Big] \boldsymbol{v}_{t+1,i}.  
\end{align}

\item[(c)] Finally, $ \tilde{\boldsymbol{\delta}}_t $ 
 indicates the difference between the utilized gradient estimation and the true gradient of the function considered at each iteration:  
\begin{equation}\label{eq: tilde_delta} 
\tilde{\boldsymbol{\delta}}_t \left( \boldsymbol{x} \right) \triangleq \nabla \tilde{f}_t \left( \boldsymbol{x} \right)  -  \tilde{\boldsymbol{G}}_t \left( \boldsymbol{x} \right). 
\end{equation}

\end{itemize}

\end{definition}

\subsection{Preliminary Lemmas}

First, note that according to the definition of $ \tilde{G}_t $ in \eqref{eq: tilde_G} and Assumption \ref{assump: unbiased 1st grad estimation + bounded var}, using 
Lemma 4.2 in \cite{gao2018information}, 
it can be shown that for all $ t \geq 0 $, 
\begin{equation}\label{eq: E [delta]=0}
\mathbb{E}_{ \boldsymbol{\zeta}_{t+1} , \boldsymbol{v}_{t+1}} \left[ \tilde{\delta}_t \left( \boldsymbol{x} \right) \right] = 0, ~~ \forall \boldsymbol{x} \in \mathbb{R}^d.
\end{equation}

In the following, some useful lemmas are in order.

\begin{lemma}\label{lem: f_tilde _t+1 - f_tilde _t}
For any $ L $-smooth function $ f \left( \boldsymbol{x} \right) $:  
$ ~\forall t \geq 0, ~ \forall \boldsymbol{x} \in \mathbb{R}^d $,
\begin{align}\label{eq: f_tilde _t+1 - f_tilde _t}
\tilde{f}_{t+1} \left( \boldsymbol{x} \right)  \leq \left\{ \begin{array}{ll}
{   \tilde{f}_{t} \left( \boldsymbol{x} \right) + \dfrac{\mu^2 L}{2}   \quad \quad }&{\mathrm{If}   \mod\left(t+1,\tau\right)=0 ~ XOR \mod\left(t,\tau\right)=0, }\\
\\
{  \tilde{f}_{t} \left( \boldsymbol{x} \right)  }&{\mathrm{Otherwise.}}
\end{array} 
\right.
\end{align}
\end{lemma}

\begin{proof}

If one and only one of $ t $ and $ t+1 $ divides $ \tau $, it can be followed that exactly one of the two functions $ \tilde{f}_{t+1} \left( .\right) $ and $ \tilde{f}_{t} \left(. \right) $ is $ f \left( .\right) $ and the other one is $ f_\mu \left( .\right) $. Therefore, according to \eqref{eq: f_mu - f}, it can be concluded that the 
difference of these two functions at any point $ \boldsymbol{x} \in \mathbb{R}^d $ is upper-bounded by $ \dfrac{\mu^2 L}{2} $. Otherwise (i.e., if both $ t $ and $ t+1 $ divide $ \tau $ or both $ t $ and $ t+1 $ do not divide $ \tau $), then according to \eqref{eq: f_mu - f}, both the associated functions are the same, and hence, their difference 
is zero at any point $ \boldsymbol{x} \in \mathbb{R}^d $.
\end{proof}

\begin{remark}
When running Algorithm \ref{alg: proposed alg} for $ N $ iterations ($ t=0,\ldots,N-1 $), it can be shown that 
if $ \tau > 1 $, the first case of the above lemma occurs exactly $ 2 \left \lfloor{\dfrac{N-1}{\tau}}\right \rfloor - 1  $ times. Otherwise (i.e., if $ \tau = 1 $), the first case never occurs. This will be used later on in upper-bounding the error terms in the proof of the main theorem, i.e., Theorem \ref{th: main}.
\end{remark}

\begin{lemma}\label{lem: norm f_tilde ^2 - norm f ^2}
For any $ L $-smooth function $ f \left( \boldsymbol{x} \right) $: 
 $ ~\forall t \geq 0, ~ \forall \boldsymbol{x} \in \mathbb{R}^d $,
\begin{align}\label{eq: norm f_tilde ^2 - norm f ^2}
\parallel \nabla  \tilde{f}_t \left( \boldsymbol{x} \right) \parallel^2 \geq   \left\{ \begin{array}{ll}
{  \parallel \nabla f \left( \boldsymbol{x}^t \right) \parallel^2   }&{\mathrm{If} \mod\left(t,\tau\right)=0, }\\
\\
{  \dfrac{1}{2} \parallel \nabla f \left( \boldsymbol{x}^t \right) \parallel^2   - \dfrac{\mu^2 d^2 L^2 }{4} \quad \quad }&{\mathrm{Otherwise.}}
\end{array} 
\right.
\end{align}
\end{lemma}

\begin{proof}
For the case when $ t $ satisfies $ \mod\left(t, \tau \right)=0 $, the result is obvious as we have  $ \tilde{f}_t \left( \boldsymbol{x} \right) = f \left( \boldsymbol{x} \right) $ (due to the definition in \eqref{eq: tilde_f}), and hence, $ \nabla \tilde{f}_t \left( \boldsymbol{x} \right) = \nabla f \left( \boldsymbol{x} \right) $. 
For the case when $ \mod\left(t, \tau \right) \neq 0 $, starting from \eqref{eq: grad f_mu - grad f} and using the fact that $ \parallel \boldsymbol{a} \parallel^2 \leq 2 \parallel \boldsymbol{b} \parallel^2 + 2 \parallel \boldsymbol{a}-\boldsymbol{b} \parallel^2 , ~ \forall \boldsymbol{a}, \boldsymbol{b} \in \mathbb{R}^d $, we have
\begin{align}
\parallel \nabla f \left( \boldsymbol{x} \right) \parallel^2 & \leq 2 \parallel \nabla \tilde{f}_t \left( \boldsymbol{x} \right) \parallel^2 + \parallel \nabla \tilde{f}_t \left( \boldsymbol{x} \right) - \nabla f \left( \boldsymbol{x} \right) \parallel^2, \notag \\
& \leq 2 \parallel \nabla \tilde{f}_t \left( \boldsymbol{x} \right) \parallel^2 + \dfrac{\mu^2 d^2 L^2}{2},
\end{align}
and consequently,
\begin{align}
\parallel \nabla \tilde{f}_t \left( \boldsymbol{x} \right) \parallel^2 \geq \dfrac{1}{2}\parallel \nabla f \left( \boldsymbol{x} \right) \parallel^2 - \dfrac{\mu^2 d^2 L^2}{4},
\end{align}
which is the desired result.

\end{proof}

\begin{lemma}\label{lem: second moment of g_tilde}
Under Assumptions \ref{assump: unbiased 1st grad estimation + bounded var}-\ref{assump: lower bound f*}, in the proposed Algorithm \ref{alg: proposed alg}, the second moment of the update direction \eqref{eq: update G tile t} can be bounded as: $ \forall t \geq 0 $, 
\begin{align}\label{eq: second moment of g_tilde}
\mathbb{E} \left[ \parallel \tilde{\boldsymbol{G}}_t \left( \boldsymbol{x}^t  \right)  \parallel^2 \right] \leq 
\left\{ \begin{array}{ll}
{ \hspace{-5 pt} \mathbb{E} \left[ \parallel \nabla f \left( \boldsymbol{x}^t \right) \parallel^2 \right]  + \dfrac{\sigma^2 }{ B m}  }&{\mathrm{If}  \mod\left(t,\tau\right)=0, }\\
\\
{ \hspace{-5 pt} \dfrac{ 2 \left( d+B m-1 \right) }{B m} ~ \mathbb{E} \left[ \parallel \nabla f \left( \boldsymbol{x}^t \right) \parallel^2 \right]   +  \dfrac{2 d \sigma^2}{B m} +\dfrac{\mu^2 L^2 d^2 }{2} \quad \quad }&{\mathrm{Otherwise,}}
\end{array} 
\right.
\end{align}

\end{lemma}

\begin{proof}
First note that according to \eqref{eq: tilde_delta} and \eqref{eq: E [delta]=0}, we have 
\begin{align}\label{eq: E tilde_G^2}
\mathbb{E}  \left[ \parallel \tilde{\boldsymbol{G}}_t \left( \boldsymbol{x}^t  \right) \parallel^2 \right]  & = \mathbb{E}_{\Omega_t} \Bigg[ \mathbb{E}_{\boldsymbol{\zeta}_{t+1} , \boldsymbol{v}_{t+1}} \left[ \parallel \tilde{\boldsymbol{G}}_t \left( \boldsymbol{x}^t  \right) \parallel^2 \left| \Omega_t \right. \right]  \Bigg] , \notag \\
& = \mathbb{E}_{\Omega_t} \Bigg[ \mathbb{E}_{\boldsymbol{\zeta}_{t+1} , \boldsymbol{v}_{t+1}} \Big[ \parallel \tilde{\boldsymbol{G}}_t \left( \boldsymbol{x}^t  \right) - \mathbb{E}_{\boldsymbol{\zeta}_{t+1} , \boldsymbol{v}_{t+1}} \left[ \tilde{\boldsymbol{G}}_t \left( \boldsymbol{x}^t  \right) \right] \parallel^2  + \parallel \mathbb{E}_{\boldsymbol{\zeta}_{t+1} , \boldsymbol{v}_{t+1}} \left[ \tilde{\boldsymbol{G}}_t \left( \boldsymbol{x}^t  \right) \right] \parallel^2 \left| \Omega_t \right. \Big] \Bigg],  \notag \\
& = \mathbb{E}_{\Omega_t} \Bigg[ \mathbb{E}_{\boldsymbol{\zeta}_{t+1} , \boldsymbol{v}_{t+1}} \left[ \parallel \tilde{\boldsymbol{G}}_t \left( \boldsymbol{x}^t  \right) - \nabla \tilde{f}_t \left( \boldsymbol{x}^t \right) \parallel^2  \left| \Omega_t \right.  \right]  + \parallel \nabla \tilde{f}_t \left( \boldsymbol{x}^t \right) \parallel^2  \Bigg],  \notag \\
& = \mathbb{E} \left[ \parallel \tilde{\boldsymbol{G}}_t \left( \boldsymbol{x}^t  \right) - \nabla \tilde{f}_t \left( \boldsymbol{x}^t \right)  \parallel^2 \right] + \mathbb{E} \left[ \parallel \nabla \tilde{f}_t \left( \boldsymbol{x}^t \right)  \parallel^2 \right] .
\end{align}

Now, note that at each iteration $ t $, one of the following two cases occurs. In the following, we expand the last equality above for each of these two cases, separately.  

\textbf{Case 1:} If $ \mod(t, \tau)=0 $, then 
substituting \eqref{eq: tilde_G} in \eqref{eq: E tilde_G^2}, we have 
\begin{align}
\mathbb{E}  \left[ \parallel \tilde{\boldsymbol{G}}_t \left( \boldsymbol{x}^t  \right)  \parallel^2 \right]  
& =  \mathbb{E} \left[ \parallel \dfrac{1}{Bm} \sum_{i=1}^m \sum_{b=1}^B \left( \nabla F \left( \boldsymbol{x}^t , \boldsymbol{\zeta}_{t+1,i,b} \right) - \nabla f \left( \boldsymbol{x}^t \right) \right) \parallel^2 \right]  + \mathbb{E} \left[ \parallel \nabla f \left( \boldsymbol{x}^t \right)  \parallel^2 \right] , \notag \\
& \stackrel{\text{(a)}}{=} \dfrac{1}{B^2 m^2} \sum_{i=1}^m \sum_{b=1}^B \mathbb{E} \left[ \parallel \nabla F \left( \boldsymbol{x}^t , \boldsymbol{\zeta}_{t+1,i,b} \right) - \nabla f \left( \boldsymbol{x}^t \right) \parallel^2 \right] + \mathbb{E} \left[ \parallel \nabla f \left( \boldsymbol{x}^t \right)  \parallel^2 \right] , \notag \\
& \stackrel{\text{(b)}}{\leq}  \dfrac{\sigma^2}{Bm} + \mathbb{E} \left[ \parallel \nabla f \left( \boldsymbol{x}^t \right)  \parallel^2 \right] ,
\end{align} 
where equation (a) results from independence of observed random samples 
$ \boldsymbol{\zeta}_{t+1,i,b}, ~ \forall i=1,\ldots,m, \forall b=1,\ldots,B $ and \eqref{eq: 1st grad unbiased} in Assumption \ref{assump: unbiased 1st grad estimation + bounded var}, and inequality (b) is due to \eqref{eq: 1st grad bounded var} in Assumption \ref{assump: unbiased 1st grad estimation + bounded var}.

\textbf{Case 2:} If $ \mod(t, \tau) \neq 0 $,  then 
substituting \eqref{eq: tilde_G} in \eqref{eq: E tilde_G^2} results in 
\begin{align}
\mathbb{E}  \left[ \parallel \tilde{\boldsymbol{G}}_t \left( \boldsymbol{x}^t  \right)  \parallel^2 \right]  & = \mathbb{E} \left[ \parallel \dfrac{1}{Bm} \sum_{i=1}^m \sum_{b=1}^B\left( \boldsymbol{G}_\mu \left( \boldsymbol{x}^t , \boldsymbol{\zeta}_{t+1,i,b} , \boldsymbol{v}_{t+1,i} \right)  - \nabla \tilde{f}_\mu \left( \boldsymbol{x}^t \right) \right) \parallel^2 \right]   + \mathbb{E} \left[ \parallel \nabla \tilde{f}_\mu \left( \boldsymbol{x}^t \right)  \parallel^2 \right] , \notag \\
& \stackrel{\text{(a)}}{=} \dfrac{1}{B^2m^2} \sum_{i=1}^m \sum_{b=1}^B  \mathbb{E} \left[ \parallel \boldsymbol{G}_\mu \left( \boldsymbol{x}^t , \boldsymbol{\zeta}_{t+1,i,b} , \boldsymbol{v}_{t+1,i} \right) - \nabla \tilde{f}_\mu \left( \boldsymbol{x}^t \right) \parallel^2 \right] + \mathbb{E} \left[ \parallel \nabla \tilde{f}_\mu \left( \boldsymbol{x}^t \right)  \parallel^2 \right] 
, \notag \\
& = \dfrac{1}{B^2m^2} \sum_{i=1}^m \mathbb{E}_{\Omega_t} \Bigg[ \mathbb{E}_{\boldsymbol{\zeta}_{t+1} , \boldsymbol{v}_{t+1}} \Big[ \parallel \boldsymbol{G}_\mu \left( \boldsymbol{x}^t , \boldsymbol{\zeta}_{t+1,i,b} , \boldsymbol{v}_{t+1,i} \right)    - \nabla \tilde{f}_\mu \left( \boldsymbol{x}^t \right)  \parallel^2 \left| \Omega_t \right. \Big]  \Bigg] \notag \\
& \quad + \mathbb{E} \left[ \parallel \nabla \tilde{f}_\mu \left( \boldsymbol{x}^t \right)  \parallel^2 \right]  , \notag \\
& \stackrel{\text{(b)}}{=} \dfrac{1}{B^2m^2} \sum_{i=1}^m \sum_{b=1}^B \mathbb{E}_{\Omega_t} \Bigg[ \mathbb{E}_{\boldsymbol{\zeta}_{t+1} , \boldsymbol{v}_{t+1}} \Big[ \parallel \boldsymbol{G}_\mu \left( \boldsymbol{x}^t , \boldsymbol{\zeta}_{t+1,i,b} , \boldsymbol{v}_{t+1,i} \right) \parallel^2 \Big] - \parallel \nabla \tilde{f}_\mu \left( \boldsymbol{x}^t \right)  \parallel^2  \Bigg] \notag \\
& \quad + \mathbb{E} \left[ \parallel \nabla \tilde{f}_\mu \left( \boldsymbol{x}^t \right)  \parallel^2 \right] ,   \notag \\
& = \dfrac{1}{B^2m^2} \sum_{i=1}^m \sum_{b=1}^B \mathbb{E}_{\Omega_t} \Bigg[ \mathbb{E}_{\boldsymbol{\zeta}_{t+1}} \Big[ \mathbb{E}_{\boldsymbol{\boldsymbol{v}_{t+1}}} \left[ \parallel \boldsymbol{G}_\mu \left( \boldsymbol{x}^t , \boldsymbol{\zeta}_{t+1,i,b} , \boldsymbol{v}_{t+1,i} \right) \parallel^2 \right] \Big]    - \parallel \nabla \tilde{f}_\mu \left( \boldsymbol{x}^t \right)  \parallel^2  \Bigg] \notag \\
& \quad + \mathbb{E} \left[ \parallel \nabla \tilde{f}_\mu \left( \boldsymbol{x}^t \right)  \parallel^2 \right] , 
\notag \\
& \stackrel{\text{(c)}}{\leq} \dfrac{1}{B^2m^2} \sum_{i=1}^m  \sum_{b=1}^B \mathbb{E}_{\Omega_t} \Bigg[ \mathbb{E}_{\boldsymbol{\zeta}_{t+1}} \Big[ 2 d  \parallel  \nabla F \left( \boldsymbol{x}^t , \boldsymbol{\zeta}_{t+1,i,b}  \right) \parallel^2 + \dfrac{\mu^2 L^2 d^2}{2} \Big]   - \parallel \nabla \tilde{f}_\mu \left( \boldsymbol{x}^t \right)  \parallel^2  \Bigg] \notag \\
& \quad + \mathbb{E} \left[ \parallel \nabla \tilde{f}_\mu \left( \boldsymbol{x}^t \right)  \parallel^2 \right] , \notag \\
& \stackrel{\text{(d)}}{=} \dfrac{1}{B^2m^2} \sum_{i=1}^m \sum_{b=1}^B \mathbb{E}_{\Omega_t} \Bigg[ 2 d ~ \mathbb{E}_{\boldsymbol{\zeta}_{t+1}} \Big[  \parallel \nabla F \left( \boldsymbol{x}^t , \boldsymbol{\zeta}_{t+1,i,b}  \right) -  \nabla f \left( \boldsymbol{x}^t \right)  \parallel^2 \Big]    + 2 d   \parallel \nabla f \left( \boldsymbol{x}^t  \right) \parallel^2 \hspace{-2 pt} + \dfrac{\mu^2 L^2 d^2}{2}  
\notag \\
& \hspace{100 pt} - \parallel \nabla \tilde{f}_\mu \left( \boldsymbol{x}^t \right)  \parallel^2   \Bigg]   + \mathbb{E} \left[ \parallel \nabla \tilde{f}_\mu \left( \boldsymbol{x}^t \right)  \parallel^2 \right] , 
\notag \\
& \stackrel{\text{(e)}}{\leq} \dfrac{1}{B^2 m^2} \sum_{i=1}^m  \sum_{b=1}^B \mathbb{E}_{\Omega_t} \Bigg[ 2 d  \sigma^2 + 2 d   \parallel \nabla f \left( \boldsymbol{x}^t  \right) \parallel^2  + \dfrac{\mu^2 L^2 d^2}{2} - \parallel \nabla \tilde{f}_\mu \left( \boldsymbol{x}^t \right)  \parallel^2  \Bigg] \hspace{-4 pt} + \mathbb{E} \left[ \parallel \nabla \tilde{f}_\mu \left( \boldsymbol{x}^t \right)  \parallel^2 \right] , 
\notag \\
& = \dfrac{2d}{Bm}  ~ \mathbb{E} \left[ \parallel \nabla f \left( \boldsymbol{x}^t  \right) \parallel^2 \right] + \dfrac{2 d  \sigma^2}{Bm} + \dfrac{\mu^2 L^2 d^2}{2Bm} + \dfrac{Bm-1}{Bm} \mathbb{E} \left[ \parallel \nabla \tilde{f}_\mu \left( \boldsymbol{x}^t \right)  \parallel^2 \right] , 
\notag \\
& \stackrel{\text{(f)}}{\leq} \dfrac{ 2 \left( d + Bm - 1 \right) }{Bm}  \mathbb{E} \left[ \parallel \nabla f \left( \boldsymbol{x}^t  \right) \parallel^2 \right] + \dfrac{2 d  \sigma^2}{Bm} + \dfrac{\mu^2 L^2 d^2}{2} , 
\end{align} 
where $ \boldsymbol{\zeta}_{t+1} \triangleq \left\lbrace \boldsymbol{\zeta}_{t+1,i,b} \right\rbrace_{i = 1,\ldots, m, b=1,\ldots,B } $, $ \boldsymbol{v}_{t+1} \triangleq \left\lbrace \boldsymbol{v}_{t+1,i} \right\rbrace_{i = 1,\ldots, m} $, and $ \Omega_{t} \triangleq \left\lbrace \boldsymbol{\zeta}_{j,i,b}, \boldsymbol{v}_{j,i} \right\rbrace_{j=0,\ldots,t , ~ i = 1,\ldots, m, ~ b=1,\ldots,B} $.
Note that equation (a) results from independence of random vectors  $ \left\lbrace \boldsymbol{\zeta}_{t+1,i,b} , \boldsymbol{v}_{t+1,i}  , \forall i=1,\ldots,m, \forall b=1,\ldots, B \right \rbrace $ 
 and equation \eqref{eq: E [delta]=0}. Moreover, equation (b) comes from \eqref{eq: tilde_delta} and \eqref{eq: E [delta]=0}. Inequality (c) comes from the definition of $ \boldsymbol{G}_\mu $ along with  Proposition 6.6. in \cite{gao2018information}.  
Furthermore, equation (d) and inequality (e) come from \eqref{eq: 1st grad unbiased} and \eqref{eq: 1st grad bounded var}  in Assumption \ref{assump: unbiased 1st grad estimation + bounded var}, respectively. Finally, inequality (f) comes from \eqref{eq: grad f_mu - grad f} and the fact that $ \parallel \boldsymbol{a} \parallel^2 \leq 2 \parallel \boldsymbol{b} \parallel^2 + 2 \parallel \boldsymbol{a}-\boldsymbol{b} \parallel^2 , ~ \forall \boldsymbol{a}, \boldsymbol{b} \in \mathbb{R}^d $.

\end{proof}

\subsection{Proof of Theorem \ref{th: main}}

Having the above 
lemmas, we are now ready to prove the main theorem. 
First note that since the objective function $ f \left( \boldsymbol{x} \right) $ is $ L $-smooth, 
its smoothing function $ f_\mu \left( \boldsymbol{x} \right) $ is also $ L $-smooth \cite{shalev2012online}.  
Therefore, it is concluded from definition \eqref{eq: tilde_f} that the function $ \tilde{f}_t \left( \boldsymbol{x} \right), ~ \forall t \geq 0 $ is also $ L $-smooth. As such, from the properties of smooth functions \cite{bottou2018optimization}, it is concluded that for any $ t \geq 0 $,
\begin{align}\label{eq: L smooth basic ineq}
\tilde{f}_t \left( \boldsymbol{x}^{t+1} \right) & \leq \tilde{f}_t \left( \boldsymbol{x}^{t} \right) + \langle \nabla \tilde{f}_t \left( \boldsymbol{x}^{t} \right) ,  \boldsymbol{x}^{t+1} - \boldsymbol{x}^{t} \rangle    + \dfrac{L}{2} \parallel \boldsymbol{x}^{t+1} - \boldsymbol{x}^{t} \parallel^2, \notag \\
& = \tilde{f}_t \left( \boldsymbol{x}^{t} \right) -  \alpha_t \langle \nabla \tilde{f}_t \left( \boldsymbol{x}^{t} \right) , \tilde{\boldsymbol{G}}_t \rangle + \dfrac{L  \alpha_t^2 }{2} \hspace{-3 pt} \parallel \tilde{\boldsymbol{G}}_t \parallel^2 \hspace{-3 pt},
\end{align}
where the  equality is due to the update equation \eqref{eq: update x} in Algorithm \ref{alg: proposed alg}. Moreover, combining \eqref{eq: f_mu - f} and \eqref{eq: tilde_f} results in
\begin{align}
\tilde{f}_{t+1} \left( \boldsymbol{x}^{t+1} \right)  \leq \tilde{f}_t \left( \boldsymbol{x}^{t+1} \right) + \dfrac{L \mu^2}{2}.
\end{align}
Therefore, summing up the above two inequalities results in
\begin{align}
\tilde{f}_{t+1} \left( \boldsymbol{x}^{t+1} \right) &\leq \tilde{f}_t \left( \boldsymbol{x}^{t} \right) -  \alpha_t \langle \nabla \tilde{f}_t \left( \boldsymbol{x}^{t} \right) , \tilde{\boldsymbol{G}}_t \rangle + \dfrac{L  \alpha_t^2 }{2} \parallel \tilde{\boldsymbol{G}}_t \parallel^2.
\end{align}
Now, by applying expectation to both sides of the above inequality, it is concluded that
\begin{align}\label{eq: tilde_f_t(x^t+1) <= tilde_f_t(x^t) +...}
\mathbb{E} \Big[  \tilde{f}_{t+1}  \left( \boldsymbol{x}^{t+1} \right) \Big] 
& \leq \mathbb{E} \left[\tilde{f}_t \left( \boldsymbol{x}^{t} \right) \right] -  \alpha_t \mathbb{E} \left[\langle \nabla \tilde{f}_t \left( \boldsymbol{x}^{t} \right) , \tilde{\boldsymbol{G}}_t \rangle \right]   + \dfrac{L  \alpha_t^2 }{2} \mathbb{E} \left[ \parallel \tilde{\boldsymbol{G}}_t \parallel^2 \right], \notag \\
& = \mathbb{E} \left[\tilde{f}_t \left( \boldsymbol{x}^{t} \right) \right] + \dfrac{L  \alpha_t^2 }{2} \mathbb{E} \left[ \parallel \tilde{\boldsymbol{G}}_t \parallel^2 \right]  -  \alpha_t \mathbb{E}_{\Omega_{t}} \left[ \mathbb{E}_{\boldsymbol{\zeta}_{t+1},\boldsymbol{v}_{t+1}} \left[\langle \nabla \tilde{f}_t \left( \boldsymbol{x}^{t} \right) , \tilde{\boldsymbol{G}}_t \rangle \Big| \Omega_{t} \right] \right] , \notag \\
& \stackrel{\text{(a)}}{=} \mathbb{E} \left[\tilde{f}_t \left( \boldsymbol{x}^{t} \right) \right] + \dfrac{L  \alpha_t^2 }{2} \mathbb{E} \left[ \parallel \tilde{\boldsymbol{G}}_t \parallel^2 \right]  -  \alpha_t \mathbb{E}_{\Omega_{t}} \left[ \langle \nabla \tilde{f}_t \left( \boldsymbol{x}^{t} \right) , \mathbb{E}_{\boldsymbol{\zeta}_{t+1},\boldsymbol{v}_{t+1}} \left[ \tilde{\boldsymbol{G}}_t \Big| \Omega_{t} \right] \rangle  \right] , \notag \\
& \stackrel{\text{(b)}}{=} \mathbb{E} \left[\tilde{f}_t \left( \boldsymbol{x}^{t} \right) \right]  + \dfrac{L  \alpha_t^2 }{2} \mathbb{E} \left[ \parallel \tilde{\boldsymbol{G}}_t \parallel^2 \right]  -  \alpha_t \mathbb{E}_{\Omega_{t}} \left[ \parallel \nabla \tilde{f}_t \left( \boldsymbol{x}^{t} \right) \parallel^2  \right],
\end{align}
where equality (a) is due to independence of $ \nabla \tilde{f}_t \left( \boldsymbol{x}^{t} \right)  $ to the random vectors $ \boldsymbol{\zeta}_{t+1} $ and $ \boldsymbol{v}_{t+1} $, and equality (b) follows from  combining \eqref{eq: tilde_delta}  and \eqref{eq: E [delta]=0}. 
Using Lemmas \ref{lem: norm f_tilde ^2 - norm f ^2} and \ref{lem: second moment of g_tilde}, we can further bound the second and third terms in the right hand side of \eqref{eq: tilde_f_t(x^t+1) <= tilde_f_t(x^t) +...}, as follows.

{\small{
\begin{align}\label{eq: v2_tilde_f_t(x^t+1) <= tilde_f_t(x^t) +...}
\mathbb{E} & \left[ \tilde{f}_{t} \left( \boldsymbol{x}^{t+1} \right) \right] 
  \leq \notag \\
&\left\{ \begin{array}{ll}
{ \hspace{-5 pt}   \mathbb{E} \left[\tilde{f}_t \left( \boldsymbol{x}^{t} \right) \right] -  \left( \alpha_t - \dfrac{L}{2}  \alpha_t^2 \right) \mathbb{E} \left[ \parallel \nabla f \left( \boldsymbol{x}^{t} \right) \parallel^2  \right] + \dfrac{ L \sigma^2 }{2 B m} \alpha_t^2  }&{ \hspace{-7 pt} \mathrm{If} \hspace{-6 pt} \mod \hspace{-5 pt} \left(t,\tau\right)=0, }\\
\\
{  \hspace{-5 pt}  \mathbb{E} \left[\tilde{f}_t \left( \boldsymbol{x}^{t} \right) \right] -  \left( \dfrac{1}{2}\alpha_t - \dfrac{\left( d+ B m-1\right)L}{ B m}  \alpha_t^2 \right) \mathbb{E} \left[ \parallel \nabla f \left( \boldsymbol{x}^{t} \right) \parallel^2  \right] + \dfrac{ \mu^2 d^2 L^2 }{4} \alpha_t \hspace{-2 pt} + \hspace{-2 pt} \left( \dfrac{d L \sigma^2}{ B m} + \dfrac{\mu^2 d^2 L^3}{4} \right) \alpha_t^2  \quad }&{\hspace{-5 pt}\mathrm{Otherwise.}}
\end{array} 
\right.
\end{align}
}}

Now, summing up the inequality \eqref{eq: v2_tilde_f_t(x^t+1) <= tilde_f_t(x^t) +...} for all $ t=0,\ldots,N-1 $, applying
 Lemma \ref{lem: f_tilde _t+1 - f_tilde _t}, and using the telescopic rule result in
\begin{align}\label{eq: f_tilde_N <= f_tilde_0 + ...}
\mathbb{E} \left[ \tilde{f}_{N} \left( \boldsymbol{x}^{N} \right) \right] & \leq \mathbb{E} \left[ \tilde{f}_{0} \left( \boldsymbol{x}^{0} \right) \right] - \sum_{t=0}^{N-1} \beta_t \mathbb{E} \left[ \parallel \nabla f \left( \boldsymbol{x}^{t} \right) \parallel^2 \right]   + \mathbf{1} \left(\tau > 1 \right)  \left( 2 \left \lfloor{\dfrac{N-1}{\tau}}\right \rfloor - 1 \right) \dfrac{\mu^2 L}{2}  + \sum_{t=0}^{N-1} A_t, 
\notag \\
& \stackrel{\text{(a)}}{\leq} \mathbb{E} \left[ \tilde{f}_{0} \left( \boldsymbol{x}^{0} \right) \right] - \sum_{t=0}^{N-1} \beta_t \mathbb{E} \left[ \parallel \nabla f \left( \boldsymbol{x}^{t} \right) \parallel^2 \right]  + \mathbf{1} \left(\tau > 1 \right) \dfrac{\mu^2 L N }{\tau}  + \sum_{t=0}^{N-1} A_t
\end{align}
where $ \mathbf{1} \left(\tau > 1 \right) $ is an indicator function which takes non-zero value of $ 1 $ only when its input argument is true; otherwise, it takes $ 0 $. Therefore, the third term in the right hand side of the above inequality does not exist if the period $ \tau  $ is set to $ 1 $.{\footnote{Note that this corresponds to the case where the worker nodes use SFO estimation in all the iterations of Algorithm \ref{alg: proposed alg}. Therefore, the proposed Algorithm \ref{alg: proposed alg} 
reduces to the method of 
 synchronous distributed  SGD described in \cite{wang2018cooperative}.}} Moreover, inequality (a) is due to the basic property of the floor function. 
Finally, for any $ t \geq 0 $, the parameters $ \beta_t $ and $ A_t $ in the above inequality are defined as  
\begin{align}\label{eq: gamma_t}
\beta_t \triangleq \left\{ \begin{array}{ll}
{ \alpha_t - \dfrac{L}{2}  \alpha_t^2 }&{\mathrm{If} \mod\left(t,\tau\right)=0, }\\
\\
{ \dfrac{1}{2}\alpha_t - \dfrac{\left( d+ B m-1\right)L}{ B m}  \alpha_t^2 \quad \quad }&{\mathrm{Otherwise,}}
\end{array} 
\right.
\end{align}
 and
\begin{align}\label{eq: X_t}
A_t \triangleq \left\{ \begin{array}{ll}
{ \dfrac{ L \sigma^2 }{2 B m} \alpha_t^2 }&{\mathrm{If} \mod\left(t,\tau\right)=0, }\\
\\
{ \dfrac{ \mu^2 d^2 L^2 }{4} \alpha_t + \left( \dfrac{d L \sigma^2}{ B m} + \dfrac{\mu^2 d^2 L^3}{4} \right) \alpha_t^2  \quad \quad }&{\mathrm{Otherwise,}}
\end{array} 
\right.
\end{align}
respectively. 

Using Definition \eqref{eq: tilde_f}, we can write inequality \eqref{eq: f_mu - f} for $ t=0 $ and $ t=N $ as follows:  
\begin{align}\label{eq: f_tilde_0(0) - f(0)}
\tilde{f}_0 \left( \boldsymbol{x}^0 \right) - f \left( \boldsymbol{x}^0 \right) \leq \dfrac{\mu^2 L}{2},
\end{align}
and
\begin{align}\label{eq: f(N) - f_tilde_N(N) }
f \left( \boldsymbol{x}^N \right) - \tilde{f}_N \left( \boldsymbol{x}^N \right) \leq \dfrac{\mu^2 L}{2},
\end{align}
respectively. 
Moreover, note that in the special case of $ \tau=1 $, we have $ \mod{\left( t,\tau \right)}=0, ~ \forall t $, and hence, according to Definition \ref{eq: tilde_f}, we have $ \tilde{f}_t \left( \boldsymbol{x} \right) - f \left( \boldsymbol{x} \right) = 0 $. As such, in this case, we can further tighten the upper-bounds in \eqref{eq: f_tilde_0(0) - f(0)} and \eqref{eq: f(N) - f_tilde_N(N) } and write them as equal to  $ 0 $. Consequently, combining these upper-bounds and then using Assumption \ref{assump: lower bound f*} results in 
\begin{align}\label{eq: f^0 - f* >= ...}
\tilde{f}_0 \left( \boldsymbol{x}^0 \right) - \tilde{f}_N \left( \boldsymbol{x}^N \right) & \leq f \left( \boldsymbol{x}^0 \right) - f \left( \boldsymbol{x}^N \right) + \mathbf{1} \left( \tau > 1 \right) \mu^2 L, \notag \\
& \leq f \left( \boldsymbol{x}^0 \right) - f^\ast + \mathbf{1} \left( \tau > 1 \right) \mu^2 L . 
\end{align}
Taking expectation from both sides of \eqref{eq: f^0 - f* >= ...} and then combining it with \eqref{eq: f_tilde_N <= f_tilde_0 + ...}  concludes that 
\begin{align}\label{eq: f_tilde_N <= f_tilde_0 + ... with gamma_t}
\sum_{t=0}^{N-1}  \beta_t ~ \mathbb{E} \left[ \parallel \nabla f \left( \boldsymbol{x}^{t} \right) \parallel^2 \right]  
 &\leq f \left( \boldsymbol{x}^0 \right)  -  f^\ast + \mathbf{1} \left(\tau > 1 \right) \left( \dfrac{\mu^2 L N }{\tau}   + \mu^2 L \right)  + \sum_{t=0}^{N-1} A_t
\end{align}

Now, choosing $ \alpha_t \leftarrow \dfrac{\sqrt{ B m}}{L \sqrt{N}} , \forall t \geq 0 $ as stated in Theorem \ref{th: main}, and defining $ \beta \triangleq \min_{t} \beta_t $, it can be derived from \eqref{eq: gamma_t} that 
\begin{align}
\beta & = \dfrac{1}{2}\alpha_t - \dfrac{\left( d+ B m-1\right)L}{ B m}  \alpha_t^2 \Bigg|_{\alpha_t =\dfrac{\sqrt{ B m}}{L \sqrt{N}}}, \notag \\
& = \dfrac{\sqrt{ B mN} - 2 \left( d+ B m-1 \right) }{ 2 L N }, \notag \\
& > \dfrac{\sqrt{ B m}}{4L \sqrt{N}}, 
\end{align}
where the inequality is due to the assumption on the number of iterations stated in Theorem \ref{th: main}. 
Therefore, we have
\begin{align}\label{eq: 1/gamma}
\dfrac{1}{\beta} \leq \dfrac{4 L N}{ \sqrt{ B mN}}.
\end{align}

Multiplying $ \dfrac{1}{\beta N} $ on both sides of inequality \eqref{eq: f_tilde_N <= f_tilde_0 + ... with gamma_t}, it follows that
\begin{align}\label{eq: f_tilde_N <= f_tilde_0 + ... with gamma}
\dfrac{1}{N}\sum_{t=0}^{N-1} \mathbb{E} \left[ \parallel \nabla f \left( \boldsymbol{x}^{t} \right) \parallel^2 \right]  \leq & \dfrac{f \left( \boldsymbol{x}^0 \right) - f^\ast}{\beta N} + \mathbf{1} \left(\tau > 1 \right)  \left( \dfrac{\mu^2 L}{\beta \tau} + \dfrac{\mu^2 L}{\beta N} \right)  + \dfrac{1}{\beta N} \sum_{t=0}^{N-1} A_t,
\end{align}
where the left hand side is the desired error term we aim to upper-bound. For this purpose, it suffices to upper-bound the last term on the right hand side. Using \eqref{eq: X_t}, we can write
\begin{align}\label{eq: X_N}
\sum_{t=0}^{N-1} A_t & = \dfrac{ L \sigma^2 }{2 B m} \sum_{t=0, \mod \left( t,\tau \right)=0}^{N-1} \alpha_t^2  + \dfrac{ \mu^2 d^2 L^2 }{4} \sum_{t=0, \mod \left( t,\tau \right) \neq 0}^{N-1} \alpha_t  + \left( \dfrac{d L \sigma^2}{ B m} + \dfrac{\mu^2 d^2 L^3}{4} \right) \sum_{t=0, \mod \left( t,\tau \right) \neq 0}^{N-1} \alpha_t^2 , \notag \\
& \stackrel{\text{(a)}}{=}  \dfrac{ L \sigma^2 }{2 L^2 N} \sum_{t=0, \mod \left( t,\tau \right)=0}^{N-1} 1  + \dfrac{ \mu^2 d^2 L \sqrt{ B m} }{4 \sqrt{N}}  \sum_{t=0, \mod \left( t,\tau \right) \neq 0}^{N-1} 1     + \left( \dfrac{d \sigma^2}{L N} + \dfrac{\mu^2 d^2 L  B m}{4 N} \right) \sum_{t=0, \mod \left( t,\tau \right) \neq 0}^{N-1} 1 ,
\notag \\
\end{align} 
where equality (a) comes from substituting the value of $ \alpha_t $. Moreover, note that the summation terms in 
the right hand side of this equality 
 can be computed 
as follows: 
\begin{align}\label{eq: number of terms of mod=0}
\sum_{t=0, \mod \left( t,\tau \right)=0}^{N-1} 1 = \left \lfloor{\dfrac{N-1}{\tau}}\right \rfloor \leq \dfrac{N}{\tau},
\end{align}
\begin{align}\label{eq: number of terms of mod neq 0}
\sum_{t=0, \mod \left( t,\tau \right) \neq 0}^{N-1} 1 = \left( N -1 - \left \lfloor{\dfrac{N-1}{\tau}}\right \rfloor \right) \leq N - \dfrac{N-1}{\tau} ,
\end{align}
where the inequalities above 
come from the basic properties of the floor function. 
Substituting \eqref{eq: number of terms of mod=0} and \eqref{eq: number of terms of mod neq 0} in \eqref{eq: X_N} concludes that
\begin{align}\label{eq: v2_X_N}
\sum_{t=0}^{N-1} A_t & = \dfrac{ L \sigma^2 }{2 L^2 N}  \left \lfloor{\dfrac{N-1}{\tau}}\right \rfloor+ \dfrac{ \mu^2 d^2 L \sqrt{ B m} }{4 \sqrt{N}} \left( N - 1 - \left \lfloor{\dfrac{N-1}{\tau}}\right \rfloor \right)+ \left( \dfrac{d \sigma^2}{L N} + \dfrac{\mu^2 d^2 L  B m}{4 N} \right) \left( N - 1 - \left \lfloor{\dfrac{N-1}{\tau}}\right \rfloor \right) , \notag \\
& 
 \leq  \dfrac{ L \sigma^2 }{2 L^2 N} \dfrac{N}{\tau}   + \mathbf{1} \left( \tau > 1 \right) \left( N -  \dfrac{N-1}{\tau} \right) \left(\dfrac{ \mu^2 d^2 L \sqrt{ B m} }{4 \sqrt{N}} + \dfrac{d \sigma^2}{L N} + \dfrac{\mu^2 d^2 L  B m}{4 N} \right)  , \notag \\
& = \dfrac{ \sigma^2 }{2 L \tau}  + \mathbf{1} \left( \tau > 1 \right) \left( N \dfrac{\tau-1}{\tau} + \dfrac{1}{\tau} \right) \left(\dfrac{ \mu^2 d^2 L \sqrt{ B m} }{4 \sqrt{N}} + \dfrac{d \sigma^2}{L N} + \dfrac{\mu^2 d^2 L  B m}{4 N} \right) , \notag \\
& = \dfrac{ \sigma^2 }{2 L \tau}  
+ \mathbf{1} \left( \tau > 1 \right)  \Bigg( \dfrac{ \mu^2 d^2 L \sqrt{ B m N} }{4} \dfrac{\tau-1}{\tau}
+ \dfrac{ \mu^2 d^2 L \sqrt{ B m} }{4 \sqrt{N} \tau}
+ \dfrac{d \sigma^2}{L} \dfrac{\tau-1}{\tau} 
+ \dfrac{d \sigma^2}{L N \tau } 
+ \dfrac{\mu^2 d^2 L  B m}{4} \dfrac{\tau-1}{\tau} \notag \\
& \hspace{342 pt}
+ \dfrac{\mu^2 d^2 L  B m}{4 N \tau} \Bigg) . 
\end{align}  

Now,  
combining \eqref{eq: v2_X_N} and \eqref{eq: 1/gamma} with \eqref{eq: f_tilde_N <= f_tilde_0 + ... with gamma},  
it follows that 
\begin{align}\label{eq: v2_f_tilde_N <= f_tilde_0 + ... with gamma}
\dfrac{1}{N} \sum_{t=0}^{N-1} \mathbb{E} \left[ \parallel \nabla f \left( \boldsymbol{x}^{t} \right) \parallel^2 \right] \leq & \dfrac{4 L \left( f \left( \boldsymbol{x}^0 \right) - f^\ast \right) }{ \sqrt{ B m N}}  
+  \dfrac{ 2 \sigma^2 }{\sqrt{ B mN} \tau}  \notag \\
& + \mathbf{1} \left( \tau > 1 \right)  \Bigg( \dfrac{4 \mu^2 L^2 N}{ \sqrt{ B m N} \tau} 
+ \dfrac{4 \mu^2 L^2}{ \sqrt{ B m N} } + \mu^2 d^2 L^2 \dfrac{\tau-1}{\tau}
+ \dfrac{ \mu^2 d^2 L^2 }{N \tau}
+ \dfrac{4 d \sigma^2}{ \sqrt{ B m N} } \dfrac{\tau-1}{\tau} 
  \notag \\
& \hspace{104 pt} + \dfrac{4 d \sigma^2}{ N \sqrt{ B m N} \tau} 
+ \dfrac{\mu^2 d^2 L^2 \sqrt{ B m} }{ \sqrt{N} } \dfrac{\tau-1}{\tau} 
+ \dfrac{\mu^2 d^2 L^2 \sqrt{ B m}}{ N \sqrt{N} \tau} \Bigg), 
\end{align}

Finally, choosing the smoothing parameter to be any value satisfying $ \mu \leq \dfrac{1}{\sqrt{dN}} $ (and consequently, $ \mu^2 \leq \dfrac{1}{d N} \leq \dfrac{1}{d \sqrt{ B mN}} $, where the last inequality comes from the fact that $ N >  B m $, 
resulted from the assumption $ N > \dfrac{16 \left( d+ B m-1 \right)^2}{ B m} $ 
 as stated in Theorem \ref{th: main}), the above inequality can be further bounded as 
\begin{align}\label{eq: v3_f_tilde_N <= f_tilde_0 + ... with gamma}
\dfrac{1}{N} \sum_{t=0}^{N-1} \mathbb{E} \left[ \parallel \nabla f \left( \boldsymbol{x}^{t} \right) \parallel^2 \right]  \leq & \dfrac{4 L \left( f \left( \boldsymbol{x}^0 \right) - f^\ast \right) }{\sqrt{ B m N}}  
+  \dfrac{ 2 \sigma^2 }{ \sqrt{ B mN} \tau} \notag \\
& + \mathbf{1} \left(\tau > 1 \right)  \Bigg( \dfrac{4 L^2}{ d \sqrt{ B m N} \tau} 
+ \dfrac{4 L^2}{ d N \sqrt{ B m N} } 
+ \dfrac{d L^2}{\sqrt{ B mN}} \dfrac{\tau-1}{\tau}
+ \dfrac{ d L^2 }{N \sqrt{ B mN} \tau}
\notag \\
& \hspace{70 pt}
+ \dfrac{4 d \sigma^2}{ \sqrt{ B m N} } \dfrac{\tau-1}{\tau}  
+ \dfrac{4 d \sigma^2}{ N \sqrt{ B m N} \tau}   
+ \dfrac{d L^2 }{ \sqrt{ B mN} } \dfrac{\tau-1}{\tau} 
+ \dfrac{d L^2 }{ N \sqrt{ B mN} \tau} \Bigg), 
\end{align}
which completes the proof of Theorem \ref{th: main}. $ \hfill \blacksquare $

\end{appendices}

\end{document}